\pdfoutput=1
\documentclass{article}
\usepackage{ijcai21}

\usepackage{times}

\usepackage{soul}
\usepackage{url}
\usepackage{breqn}
\usepackage[utf8]{inputenc}
\usepackage[small]{caption}
\usepackage{graphicx}
\usepackage{amsmath}
\usepackage{booktabs}
\usepackage{xspace}
\usepackage[shortlabels]{enumitem}
\usepackage{amsthm}

\usepackage{cleveref}
\usepackage{xcolor}

\newcommand{\citet}[1]{\citeauthor{#1}~\shortcite{#1}}
\newcommand{\citep}{\cite}

\usepackage{amsmath,amsfonts,bm}

\def\eqref#1{equation~\ref{#1}}

\def\1{\bm{1}}

\DeclareMathAlphabet{\mathsfit}{\encodingdefault}{\sfdefault}{m}{sl}
\SetMathAlphabet{\mathsfit}{bold}{\encodingdefault}{\sfdefault}{bx}{n}

\title{Evaluating Relaxations of Logic for Neural Networks: \\A Comprehensive Study}

\author{
Mattia Medina Grespan \and 
Ashim Gupta \And 
Vivek Srikumar\\
\affiliations
University of Utah\\
\emails 
\{mattiamg,ashim,svivek\}@cs.utah.edu
}

\definecolor{orange}{rgb}{1,0.5,0}
\definecolor{mdgreen}{rgb}{0.05,0.6,0.05}
\definecolor{mdblue}{rgb}{0,0,0.7}
\definecolor{dkblue}{rgb}{0,0,0.5}
\definecolor{dkgray}{rgb}{0.3,0.3,0.3}
\definecolor{slate}{rgb}{0.25,0.25,0.4}
\definecolor{gray}{rgb}{0.5,0.5,0.5}
\definecolor{ltgray}{rgb}{0.7,0.7,0.7}
\definecolor{purple}{rgb}{0.7,0,1.0}
\definecolor{lavender}{rgb}{0.65,0.55,1.0}
\definecolor{mdmagenta}{rgb}{1,0,1}

\newcommand{\luka}{\L{}ukasiewicz\xspace}
\newcommand{\godel}{G\"odel\xspace}
\newcommand{\tnorm}{t-norm\xspace}
\newcommand{\tnorms}{t-norms\xspace}

\newcommand{\boolean}[1]{\texttt{#1}}
\newcommand{\pred}[1]{\boolean{#1}}

\newtheorem{definition}{Definition}
\newtheorem{example}{Example}
\newtheorem{proposition}{Proposition}

\newcommand{\p}[1]{\left(#1\right)}

\newcommand{\pb}[1]{\left[#1\right]}
\newcommand{\pc}[1]{\left\{#1\right\}}
\newcommand{\mbool}[1]{\text{\boolean{}}}

\begin{document}

\maketitle

\begin{abstract}

    Symbolic knowledge can provide crucial inductive bias for training neural models, especially in low data regimes. A successful strategy for incorporating such knowledge involves relaxing logical statements into sub-differentiable losses for optimization. 
    In this paper, we study the question of how best to relax logical expressions that represent labeled examples and knowledge about a problem; we focus on sub-differentiable \tnorm relaxations of logic. 
    We present theoretical and empirical criteria for characterizing which relaxation would perform best in various scenarios. In our theoretical study driven by the goal of preserving tautologies, the \luka \tnorm performs best. However, in our empirical analysis on the text chunking and digit recognition tasks, the product \tnorm achieves best predictive performance. We analyze this apparent discrepancy, and conclude with a list of best practices for defining loss functions via logic.

\end{abstract}

\section{Introduction}
\label{sec:intro}
Neural networks are remarkably effective across many domains and tasks; but their usefulness is limited by their data hungriness. A promising direction towards alleviating this concern involves augmenting learning with rules written in first-order logic~\citep[inter alia]{rocktaschel2015injecting,li2019augmenting,li2019logic-driven,fischer2019dl2}. 
To make rules amenable with gradient-based learning, this approach calls for relaxing the logical operators to define sub-differentiable loss terms.
A systematic method to perform this relaxation
uses a well-studied family of binary operators, namely \emph{triangular norms} or \emph{t-norms}~\citep{klement2013triangular}. 
Different t-norms define different $[0,1]$-valued interpretations of the Boolean operators.

There are infinitely many such t-norm logics, but the most commonly used ones are: Product~\citep{rocktaschel2015injecting,li2019logic-driven,asai2020logicguided}, \godel~\citep{minervini2017adversarial} and \luka~\citep{bach2017hinge}. 
While the usefulness of such relaxations is well established, the questions of how they compare against each other, and even how such a comparison should be defined remain open.

This paper is a first step towards answering these important questions by analyzing the three t-norm relaxations and their variants. To do so, we define the criteria for quantifying the goodness of a t-norm based relaxation. On the theoretical side, we rank logic relaxations by their \emph{consistency}, i.e., their ability to preserve the truth of tautologies. On the empirical side, we use a principled approach to construct loss functions using t-norms that subsume traditional loss functions like cross-entropy, and study how well different relaxations compare with standard gradient-based learning. We report the results of experiments on two tasks: jointly recognizing digits and predicting the results of arithmetic operators, and text chunking. Both the theoretical and empirical criteria concur in the recommendation that a variant of the Product t-norm ($\mathcal{R}$-Product) is most suitable for introducing logical rules into neural networks.

In summary, the main contributions of this work are:

\begin{itemize}[nosep]
    \item We define theoretical and empirical properties that any relaxation of logic should have to be useful for learning.
    
    \item We define the consistency of relaxations to rank them in terms of their ability to preserve tautologies.
    
    \item We present empirical comparisons of logic relaxations on two tasks where labeled examples and declaratively stated knowledge together inform neural models.
\end{itemize}

\section{Problem Statement and Notation}
\label{sec:problem}
Several recent efforts have shown the usefulness of declarative knowledge to guide neural network learning towards improving model quality. While different approaches exist for incorporating such rules into neural models~\citep[for example]{xu2018semantic}, a prominent strategy involves relaxing logic to the real regime using the well-studied \tnorm relaxations.

\textbf{Triangular norms} (\tnorms) arose in the context of probabilistic metric spaces~\citep{menger1942statistical},
and for our purposes, represent a relaxation of the Boolean conjunction which agrees with the definition of conjunctions for $\{0,1\}$-inputs. Given such a relaxation and a relaxation of $\neg X$ as $1-x$, we have two axiomatic approaches for defining implications. The first (called $\mathcal{S}$-logics) treats implications as  disjunctions (i.e., $X \to Y = \neg X \vee Y$), while the second (called $\mathcal{R}$-logics) defines implications axiomatically. We refer the reader to~\citet{klement2013triangular} for a detailed treatment of \tnorms. 

Table~\ref{tab:tnorms_definition} shows the set of t-norms we consider here; these have been used in recent literature to inject knowledge into neural networks.
However, despite their increasing prevalence, there is no consensus on which \tnorm to employ. A survey of recent papers reveals the use of  $\mathcal{S}$-Product~\citep{rocktaschel2015injecting},  $\mathcal{R}$-Product~\citep{li2019logic-driven,asai2020logicguided}, \luka~\citep{bach2017hinge}, $\mathcal{S}$-\godel~\citep{minervini2017adversarial} and even a mixture of the $\mathcal{S}$-\godel and $\mathcal{R}$-Product~\citep{li2020structured} \tnorms. 

\emph{How do these relaxations of logic compare against each other?}
In this work, we answer this question from both theoretical and empirical perspectives. 

\subsection{Learning From Logic: The Setup}\label{learning_logic}
To compare the different relaxations of logic on an even footing, let us first see a general recipe for converting logic rules to loss functions for a given relaxation\footnote{The setup described here is implicitly present in~\citet{rocktaschel2015injecting},~\citet{li2019logic-driven}, and others.}. We will use the task of recognizing handwritten digits as a running example.

Given a labeling task, we can represent the fact that the label for an instance $x$ is $y$ as a predicate, say $\pred{Label}(x,y)$. In our running example, we could define a predicate $\pred{Digit}(x,y)$ to denote that an image $x$ represents the digit $y$. We could also define a predicate $\pred{Sum}(x_1, x_2, y)$ to denote the fact that the digits in two images $x_1$ and $x_2$ add up to the digit $y~(\text{mod} 10)$. From this perspective, we can treat classifiers as predicting probabilities that the predicates hold.

To train such classifiers, we typically have a training set $D$ of labeled examples $(x, y)$. In our notation, each such example can be represented as the predicate $\pred{Label}(x,y)$ and the training set is a conjunction of such predicates
\begin{align}
    \bigwedge_{(x, y) \in D} \pred{Label}(x, y)
    \label{eq:dataset}
\end{align}

Sometimes, instead of a labeled dataset, we may have a constraint written in logic. In our running example, from the definition of the \pred{Digit} and \pred{Sum} predicates, we know that 
\begin{align}
    &\forall x_1, x_2,~\bigwedge_{y_1, y_2}~\pred{Digit}\p{x_1, y_1} \wedge \pred{Digit}\p{x_2, y_2} \rightarrow \nonumber\\
    &\quad\quad\quad\quad\quad\quad\quad\pred{Sum}\p{x_1, x_2, (y_1 + y_2)~\text{mod}~10} \label{eq:digit-sum-constraint}
\end{align}

Note that such a constraint need not depend on labeled examples, and should hold irrespective of what labels the examples should be assigned. In general, given a large unlabeled set of examples denoted by ${x} \in U$, we can write constraints 
\begin{align}
    \bigwedge_{x \in U} \pred{C}\p{x}
    \label{eq:generic-constraint}
\end{align}
These constraints may be composite formulas constructed with predicates as shown in our running example above.

From this standpoint, we can envision the goal of learning as that of ensuring that the formulas representing labeled examples (~\eqref{eq:dataset}) and constraints (~\eqref{eq:generic-constraint}) hold. Since we are treating classifiers as predicting probabilities that the atomic predicates hold, we can equivalently state the learning problem as that of finding model parameters that maximize the value of a \emph{relaxation} of the conjunction of the formulas representing the data and constraints. In other words, we can use logic to define loss functions.

In this declarative learning setting, we have the choice of using \emph{any} models (e.g., CNNs) for our predicates, and \emph{any} relaxation of logic. If we only have labeled examples, and we use one of the Product relaxations, we recover the widely used cross-entropy loss~\citep{li2019logic-driven,DBLP:conf/ilp/GianniniMDMG19}.

\begin{table}
\setlength{\tabcolsep}{3pt}
\begin{tabular}{r|cccc}
\toprule
 & $\mathcal{S}$-\godel & $\mathcal{R}$-Product &  \luka\\
 \midrule
 $\wedge$  & $\min(x,y)$ & $x\cdot y$  & $\max(0,x+y-1)$ \\
 $\neg$ & $1-x$  & $1-x$ 
     & $1-x$  \\
$\vee$  & $\max(x,y)$ & $x+y-x\cdot y$  & $\min(1,x+y)$ \\
 $\rightarrow$ &  $\max(1-x,y)$ & 
  $\left\{
         \begin{array}{ll}
             1 & \text{if }x \leq y \\
             \frac{y}{x} & \text{otherwise}
         \end{array}
     \right.$ 
     & $\min(1,1-x+y)$ \\
      \bottomrule
\end{tabular}

\caption{T-norm relaxations studied in this work. Here, the letters $x$ and $y$ denote the relaxed truth values of the arguments of the formulas. In the implication definitions, $x$ and $y$ denote the antecedent and the consequent respectively. The table does not show $\mathcal{S}$-Product: it agrees with $\mathcal{R}$-Product for all the connectives except the implication, defined as $1-x+x\cdot y$. We are defining $\mathcal{R}$-Product using its SBL$_\sim$ extension with involutive negation (See ~\protect\citet{esteva2000residuatedinvolutive}).}
\label{tab:tnorms_definition}
\end{table}

\paragraph{Notation.} We use upper case letters (e.g., $\pred{P}$, $\pred{Digit}$) to represent Booleans, and lower cased letters (e.g., $p$, $digit$) to represent their relaxations. In some places, for clarity, we use square brackets to denote the relaxation of a Boolean formula \boolean{A} (i.e., [\boolean{A}]$=$a).

\section{Validity of Relaxed Logic}
\label{sec:desiderata}

In this section, we propose three criteria that a logic relaxation should satisfy to be useful for learning.

\paragraph{Consistency.}
\label{sec:consistency} The language of logic can declaratively introduce domain knowledge, invariants, or even reasoning skills into neural networks. However, to admit reasoning, tautologies should always hold. That is, the truth value of any tautology should be $1$ irrespective of the value of its constituent atomic predicates. Equivalently, the integral of the relaxation of a tautology over the domain of its atomic predicates should be $1$. We can formalize this intuition.

\begin{definition}\label{consistency}
Let \boolean{T} be a tautology in predicate logic formed with a set of atomic predicates $\mathcal{T}$, and let $L$ be a logic relaxation. The \emph{consistency} of \boolean{T} in $L$, denoted as $\kappa^L(\text{\boolean{T}})$, is defined as 
\begin{align}
    \kappa^L(\text{\boolean{T}})=\int_0^1[\text{\boolean{T}}]\hspace{0.1cm}\mathrm{d}\mathcal{T} \label{eq:consistency-definition}
\end{align}
\end{definition}
If the consistency $\kappa^L(\text{\boolean{T}})= 1$, we will say that the tautology \boolean{T} is consistent under the relaxation $L$.

\paragraph{Self-consistency.} Every Boolean statement implies itself. That is, the statement $\pred{P} \leftrightarrow \pred{P}$ is a tautology for any $\pred{P}$. This observation gives us the definition of \emph{self-consistency} of a formula under a given relaxation.

\begin{definition} Let {\boolean{P}} be any Boolean formula in predicate logic with a set of atomic predicates $\mathcal{P}$ , and let $L$ be a logic relaxation. The \emph{self-consistency} of $\mathcal{P}$ in the logic $L$, denoted as $\kappa_S^L(\text{\boolean{P}})$, is defined as 
\begin{align}
    \kappa_S^L(\text{\boolean{P}})=\kappa^L(\text{\boolean{P}}\leftrightarrow     \text{\boolean{P}})=\int_0^1[\text{\boolean{P}}\leftrightarrow \text{\boolean{P}}]\hspace{0.1cm}\mathrm{d}\mathcal{P}
    \label{eq:self-consistency-definition}
\end{align}
\end{definition}

If $\kappa_S^L(\text{\boolean{P}})\neq 1$ we will say that formula \boolean{P} is not self-consistent under a relaxation $L$.
Since we consider a dataset to be a conjunction of facts (~\eqref{eq:dataset}), the self-consistency of large conjunctions allows us to judge whether a dataset implies itself under a relaxation.

\paragraph{Sub-differentiability.}\label{sec:properties} Since our eventual goal is to relax declaratively stated knowledge to train neural networks, the relaxations should admit training via backpropagation. As a result, the functions defining the relaxed logical operators should at least be sub-differentiable.

In sum, we consider a logic relaxation to be valid if the following properties hold:
\begin{enumerate}
    \item[(P1)] It must be sub-differentiable over the interval $[0,1]$. 
    \item[(P2)] It must be consistent for any tautology.
    \item[(P3)] It must be self-consistent for any Boolean formula.
\end{enumerate}
All the relaxations we study here satisfy property P1\footnote{This is not always the case. In $\mathcal{R}$-\godel logics, for instance, implications are not sub-differentiable:  $\pb{\boolean{X}\rightarrow\boolean{Y}}$,  takes value 1 if $y\geq x$, and $y$ otherwise.}. Even among sub-differentiable relaxations, some may be easier than others to learn using gradient-based approaches. We consider this empirical question in~\S \ref{sec:experiments}. Properties P2 and P3 do not always hold, and we will prefer relaxations that have higher values of consistency and self-consistency.

\begin{table}[tp]
    
    \resizebox{\columnwidth}{!}{
    \setlength{\tabcolsep}{1.1pt}
    \begin{tabular}{lrrrr}\toprule
      Tautologies & $\mathcal{S}$\text{-Prod.} & $\mathcal{S}$\text{-\godel} & \L{}uka. & $\mathcal{R}$\text{-Prod.} \\ \midrule
    \begin{tabular}{@{}l@{}}\textbf{Axiom Schemata} \\ \quad $\text{\boolean{P}}\rightarrow (\boolean{Q}\rightarrow \boolean{P})$\end{tabular}          & \begin{tabular}{@{}c@{}}\\ 0.92\end{tabular}                    & \begin{tabular}{@{}c@{}}\\ 0.79\end{tabular} & \begin{tabular}{@{}c@{}}\\ 1\end{tabular} & \begin{tabular}{@{}c@{}}\\ 1\end{tabular} \\ 
     \quad $(\boolean{P}\rightarrow(\boolean{Q}\rightarrow \boolean{R}))\rightarrow((\boolean{P}\rightarrow \boolean{Q})\rightarrow(\boolean{P}\rightarrow \boolean{R}))$			  & 0.88		    & 0.75		& 0.96	 & 0.93	 \\ 
	\quad$(\neg \boolean{P}\rightarrow\neg \boolean{Q})\rightarrow(\boolean{Q}\rightarrow \boolean{P})$		  & 0.86		    & 0.75& 1	 & 0.88 \\ 
	\midrule
	\begin{tabular}{@{}l@{}}\textbf{Primitive Propositions} \\ \quad$(\boolean{P}\vee \boolean{P})\rightarrow \boolean{P}$\end{tabular}		  & \begin{tabular}{@{}c@{}}\\ 0.75\end{tabular}	& \begin{tabular}{@{}c@{}}\\ 0.75\end{tabular}	& \begin{tabular}{@{}c@{}}\\ 0.75\end{tabular}	 & \begin{tabular}{@{}c@{}}\\ 0.69\end{tabular} 	  \\ 
	\quad$\boolean{Q}\rightarrow(\boolean{P}\vee \boolean{Q})$  & 0.92		    & 0.79	& 1	 & 1 \\ 
	\quad$(\boolean{P}\vee \boolean{Q})\rightarrow(\boolean{Q}\vee \boolean{P})$	 & 0.86	& 0.75		& 1	 & 1\\ 
	\quad$(\boolean{P}\vee( \boolean{Q}\vee  \boolean{R} ))\rightarrow( \boolean{Q}\vee(\boolean{P}\vee \boolean{R}))$	  & 0.91	& 0.78	& 1	 & 1 \\ 
	\quad$( \boolean{Q}\rightarrow \boolean{R})\rightarrow((\boolean{P}\vee  \boolean{Q})\rightarrow(\boolean{P}\vee \boolean{R}))$ & 0.90		    & 0.76 & 1	 & 1  \\ 
	\begin{tabular}{@{}l@{}}\textbf{Law of excluded middle} \\ \quad$\boolean{P}\vee\neg \boolean{P}$\end{tabular}	& \begin{tabular}{@{}c@{}}\\ 0.83\end{tabular} & \begin{tabular}{@{}c@{}}\\ 0.75\end{tabular}& \begin{tabular}{@{}c@{}}\\ 1\end{tabular} & \begin{tabular}{@{}c@{}}\\ 0.83\end{tabular}\\ 
	\midrule
	\begin{tabular}{@{}l@{}}\textbf{Law of contradiction} \\ \quad$\neg(\boolean{P}\wedge\neg \boolean{P})$\end{tabular}			  & \begin{tabular}{@{}c@{}}\\ 0.83\end{tabular}	& \begin{tabular}{@{}c@{}}\\ 0.75\end{tabular}	& \begin{tabular}{@{}c@{}}\\ 1\end{tabular} & \begin{tabular}{@{}c@{}}\\ 0.83\end{tabular}\\ 
	\begin{tabular}{@{}l@{}}\textbf{Law of double negation} \\ \quad$\boolean{P}\leftrightarrow\neg(\neg \boolean{P})$\end{tabular}  &\begin{tabular}{@{}c@{}}\\ 0.70\end{tabular}   & \begin{tabular}{@{}c@{}}\\ 0.75\end{tabular} 	& \begin{tabular}{@{}c@{}}\\ 1\end{tabular} 	 & \begin{tabular}{@{}c@{}}\\ 1\end{tabular} \\ 
	\midrule
	\begin{tabular}{@{}l@{}}\textbf{Principles of transposition} \\ \quad$(\boolean{P}\leftrightarrow \boolean{Q})\leftrightarrow (\neg \boolean{P}\leftrightarrow\neg \boolean{Q})$\end{tabular}   & \begin{tabular}{@{}c@{}}\\ 0.61\end{tabular}   & \begin{tabular}{@{}c@{}}\\ 0.67\end{tabular} & \begin{tabular}{@{}c@{}}\\ 1\end{tabular}  & \begin{tabular}{@{}c@{}}\\ 0.59\end{tabular} \\ 
	\midrule
	\begin{tabular}{@{}l@{}}\textbf{Laws of tautology} \\ \quad$\boolean{P}\leftrightarrow (\boolean{P}\wedge \boolean{P})$\end{tabular}  & \begin{tabular}{@{}c@{}}\\ 0.69\end{tabular}   & \begin{tabular}{@{}c@{}}\\ 0.75\end{tabular}& \begin{tabular}{@{}c@{}}\\ 0.75 \end{tabular}& \begin{tabular}{@{}c@{}}\\ 0.50\end{tabular} \\ 
	\quad$\boolean{P}\leftrightarrow (\boolean{P}\vee \boolean{P})$  & 0.69 & 0.75	& 0.75	 & 0.69  \\ 
	\begin{tabular}{@{}l@{}}\textbf{De Morgan's Laws} \\ \quad$(\boolean{P}\wedge \boolean{Q})\leftrightarrow\neg(\neg \boolean{P}\vee \neg \boolean{Q})$\end{tabular} & 0.75 & 0.75 & 1 & 1 \\
	\quad$\neg(\boolean{P}\wedge \boolean{Q})\leftrightarrow\neg(\neg \boolean{P}\vee \neg \boolean{Q})$ & 0.75 & 0.75 &1 & 1\\
	\bottomrule
\end{tabular}}
\caption{Consistencies of a (subset) of representative set of tautologies under different logic relaxations. We see here that the $\mathcal{R}$-Product and \luka relaxations are generally more consistent, suggesting that they are preferable to the other two relaxations. Other tautologies we examined show the same trends.}
\label{tab:tautotable}
\end{table}

\section{Truth Preservation of Relaxations}
\label{sec:truthpres}

In this section, we will assess the relaxations from Table~\ref{tab:tnorms_definition} with respect to properties P2 and P3.

\subsection{Consistency}\label{consistency_subsection}
Property P2 expects \textit{every} tautology to be consistent under a valid logic relaxation. Since predicate logic admits infinitely many tautologies, we will use a representative set of tautologies for our evaluation. This set contains the \textit{Axiom Schemata} of the Hilbert proof system for predicate logic\footnote{This choice is motivated because, loosely speaking, every tautology is generated using the Axiom Schemata with the \emph{modus ponens} proof rule.}, the primitive propositions, and a set of elementary properties defined by~\citet{russell1910principia}.

Table~\ref{tab:tautotable} shows the consistency for each tautology for our \tnorm relaxations. Comparing results across the columns, we see that \luka and $\mathcal{R}$-Product \tnorm logics \textit{preserve truth} the most across our representative set. In general, we find $\mathcal{R}$ logics to be better at preserving truth than $\mathcal{S}$ logics.

\begin{example} Consistency of the tautology $\boolean{A}\rightarrow\boolean{A}$ using $\mathcal{S}$-Product. 
 \begin{align*}
    [\boolean{A}\rightarrow \boolean{A}] &=1-a+a^2
 \end{align*}
By the definition of consistency (~\eqref{eq:consistency-definition}), we have
 \begin{align*}
     \kappa^{\mathcal{S}\text{-Product}}(\boolean{A}\rightarrow \boolean{A}) =
    & \int_0^1 1-a+a^2\hspace{.1cm}\,\mathrm{d}a = \frac{5}{6}\approx0.83
 \end{align*}
\end{example}

\subsection{Self-consistency}
The validity property P3 states that \textit{every} well-formed Boolean formula should be self-consistent for a relaxation of the logic to be valid. It turns out that the definition of implications for any $\mathcal{R}$ logic (including \luka) guarantees the self-consistency of every formula.

\begin{proposition}\label{prop:R-is-self-consist}
Every formula is self-consistent under any $\mathcal{R}$-logic relaxation.
\end{proposition}

This follows directly from the definition of the  \tnorm and the properties of residua. The appendix has a proof sketch.

However, the same is not true for $\mathcal{S}$-logics. For example, for the conjunction $\boolean{P}=\boolean{A}\wedge\boolean{B}$, the self-consistency under the \godel relaxation $\kappa_S^{\mathcal{S}\text{-\godel}}(\boolean{P}) = 0.75$ and under the Product relaxation, we have $\kappa_S^{\mathcal{S}\text{-Product}}(\boolean{P})\approx 0.74$.
These results suggest that the $\mathcal{R}$-Product and \luka relaxations are preferable from the perspective of property P3 as well.

Intriguingly, we find that the $\mathcal{S}$-Product logic is \emph{eventually} self-consistent for large monotone conjunctions. 

\begin{proposition}\label{prop:selfconprod}
Let $A=\bigwedge_{i=1}^n \boolean{A}_i$ be a conjunction consisting of $n$ atomic predicates $\boolean{A}_1, \boolean{A}_2, \dots, \boolean{A}_n$. The self-consistency of $A$ is given by $\kappa_S^{\mathcal{S}\text{-Product}}(A)=1-\frac{2}{2^n}+\frac{3}{3^n}-\frac{2}{4^n}+\frac{1}{5^n}$.
\end{proposition}
\begin{proof}
By induction over the size of the conjunction $n$. 
\end{proof}

We see that, as $n\to\infty$, the self-consistency of a monotone conjunction approaches $1$. In other words, large conjunctions (e.g., representing large datasets) are essentially self-consistent under the $\mathcal{S}$-Product relaxation.

\section{Empirical Comparisons of Relaxed Logic}
\label{sec:experiments}
In this section, we empirically study the differences between the different logic relaxations using two tasks: recognizing digits and arithmetic operations, and text chunking.
In both tasks, we set up the learning problem in terms of logic, and compare models learned via different logic relaxations.\footnote{Our PyTorch~\citep{paszke2019pytorch} code is archived at \url{https://github.com/utahnlp/neural-logic}}

\subsection{Recognizing Digits and Arithmetic Operations}\label{subsec:digits}

These experiments build upon our running example from~\S \ref{sec:problem}. We seek to categorize handwritten digits; i.e., we learn the predicate $\boolean{Digit}$. In addition, we also seek to predict the sum and product (modulo 10) of two handwritten digit images. These correspond to the predicate $\boolean{Sum}$ we have seen, and a new analogous predicate $\boolean{Product}$.

We use the popular MNIST dataset~\citep{lecun1998mnist} for our experiments, but \emph{only} to supervise the $\boolean{Digit}$ classifier. Rather than directly supervising the other two classifiers, we use coherence constraints over \emph{unlabeled} image pairs that connect them to the $\boolean{Digit}$ model. The constraint for the $\boolean{Sum}$ classifier is shown in~\eqref{eq:digit-sum-constraint}, and the one for the $\boolean{Product}$ classifier is similarly defined.

We set up the learning problem as defined in~\S\ref{sec:problem} and compare performance across different relaxations.

\subsubsection{Data and Setup}
We partition the 60k MNIST training images into TRAIN and DEV sets, with 50k and 10k images respectively. To supervise the \boolean{Digit} model, we sample 1k, 5k and 25k labeled images from TRAIN to form three DIGIT sets. The coherence constraints are grounded in 5k \emph{unlabeled} image pairs consisting of images from TRAIN that are not in any DIGIT set, giving us the PAIR dataset. 

For evaluating the \boolean{Digit} model, we use the original 10k TEST examples from MNIST. For the development and evaluation of the operator models, we sample random image pairs from DEV and TEST to create the PairDEV and PairTEST sets respectively. The ground truth \boolean{Sum} and \boolean{Product} labels for these image pairs can be computed by the sum and product modulo 10 of the image labels.

We use CNNs as the \boolean{Digit}, \boolean{Sum} and \boolean{Product} models.
For the operator models, we concatenated the two images to get CNN inputs. 
To jointly train these models using the labeled DIGIT data and the unlabeled PAIR datasets, we define a loss function by relaxing the conjunction of \boolean{Digit} predicates over the DIGIT examples and the coherence constraints over the PAIR examples. That is, learning requires minimizing:

\begin{dmath}
    -\pb{\p{\bigwedge_{(x,y) \in D}\texttt{Digit}(x,y)}\wedge\p{\bigwedge_{PAIR} \texttt{Sum Coherence}}\wedge\p{\bigwedge_{PAIR}\texttt{Product Coherence}}}
\end{dmath}

In practice, we found that it is important to use a hyperparameter $\lambda$ that weights the relaxed coherence constraints in the loss. We used the DEV sets for hyperparameter tuning using the average of the accuracy of the \texttt{Digit} classifier and the coherences of the other two.

\subsubsection{Results} 
Table~\ref{tab:joint_digit} reports accuracies for the \texttt{Digit} classifier trained with different sizes of DIGIT, and the coherence constraints instantiated over the 5k PAIR examples. We observe that the $\mathcal{R}$-Prod relaxation dominates across all settings, with higher gains when there are fewer labeled examples. (The bold entries in this and other tables are statistically significantly better than the other relaxations at $p<0.05$. The accuracies are averages from three runs with different seeds along with the standard deviation.)
Table~\ref{tab:joint_ope} reports the average of \texttt{Sum}, and \texttt{Prod} accuracies for the same settings, and Table~\ref{tab:joint_coherence} shows the fraction of PairTEST examples where the coherence constraints are satisfied. From these results, we see that the $\mathcal{R}$-Product and  $\mathcal{S}$-\godel relaxations offer the best accuracies.

Interestingly, \luka is the least accurate relaxation. However, the losses compiled using \luka and $\mathcal{S}$-\godel \tnorms were unstable, and the results shown here were achieved with additional assumptions. We defer this technical discussion to~\S \ref{sec:theory-vs-experiments}, and conclude that the stability and accuracy of the $\mathcal{R}$-Product relaxation suggest that it is the most suitable relaxation for this class of problems.

\begin{table}[tp]
    \centering
    \begin{tabular}{rccc}
    \toprule
         & 1000 & 5000 & 25000 \\
    \midrule
     $\mathcal{S}$-\godel & 95.0 {\tiny (0.3)} & 97.4 {\tiny (0.1)} & 97.7 {\tiny (0.1)} \\
     $\mathcal{S}$-Product & 95.1 {\tiny (0.1)} & 98.2 {\tiny (0.0)} & 99.0 {\tiny (0.0)}\\
     $\mathcal{R}$-Product & \textbf{96.3} {\tiny (0.1)} & \textbf{98.4} {\tiny (0.1)} & \textbf{99.2} {\tiny (0.0)} \\
     \luka & 95.8 {\tiny (0.1)}& 98.0 {\tiny (0.1)} & 99.1 {\tiny (0.0)}\\
     \bottomrule
    \end{tabular}
    \caption{\boolean{Digit} accuracies (and standard deviations) from jointly training \boolean{Digit} on DIGIT sizes 1k, 5k, 25k and operators (\boolean{Sum},\boolean{Prod}) on PAIR size 5k.}
    \label{tab:joint_digit}
\end{table}

\begin{table}[tp]
    \centering
    \begin{tabular}{rccc}
    \toprule
         & 1000 & 5000 & 25000 \\
    \midrule
     $\mathcal{S}$-\godel & 87.3 {\tiny(0.5)} & \textbf{91.1} {\tiny(0.1)} & 91.5 {\tiny(0.0)} \\
     $\mathcal{S}$-Product & 76.9 {\tiny(0.8)} & 88.6 {\tiny(0.6)} & 90.1 {\tiny(0.0)} \\
     $\mathcal{R}$-Product & \textbf{88.0} {\tiny(0.3)} & 90.8 {\tiny(0.3)} & \textbf{91.8} {\tiny(0.0)} \\
     \luka & 75.9 {\tiny(3.1)} & 84.5 {\tiny(2.8)} & 82.3 {\tiny(0.3)}\\
     \bottomrule
    \end{tabular}
    \caption{Average of \boolean{Sum} and \boolean{Product} accuracies (and standard deviations) from jointly training \boolean{Digit} on DIGIT sizes 1k, 5k, 25k and operators on PAIR size 5k.}
    \label{tab:joint_ope}
\end{table}

\begin{table}[tp]
    \centering
    \begin{tabular}{rccc}
    \toprule
         & 1000 & 5000 & 25000 \\
    \midrule
     $\mathcal{S}$-\godel & 86.4 {\tiny(0.6)} & 91.4 {\tiny(0.1)}  & 91.9 {\tiny(0.1)}  \\
     $\mathcal{S}$-Product & 79.0 {\tiny(0.7)} & 89.3 {\tiny(0.5)} & 90.3 {\tiny(0.0)} \\
     $\mathcal{R}$-Product & \textbf{89.1} {\tiny(0.3)} & \textbf{91.5} {\tiny(0.4)} & \textbf{92.1} {\tiny(0.1)} \\
     \luka & 77.4 {\tiny(3.0)} & 85.4 {\tiny(2.7)} & 82.6 {\tiny(0.2)} \\
     \bottomrule
    \end{tabular}
    \caption{Average of \boolean{Sum} and \boolean{Prod} Coherence accuracies (and standard deviations) from jointly training \boolean{Digit} on DIGIT sizes 1k, 5k, 25k and operators on PAIR size 5k.}
    \label{tab:joint_coherence}
\end{table}

\subsubsection{Joint Learning vs. Pipelines} 
In the coherence constraints in \eqref{eq:digit-sum-constraint}, if we knew both the \boolean{Digit} terms, we can deterministically compute the value of \boolean{Sum}. This suggests a pipeline strategy for training, where we can train the \boolean{Digit} classifier alone, and use it to assign (noisy) labels to the unlabeled PAIR data. Subsequently, we can train the \boolean{Sum} and \boolean{Product} models independently. How does the joint training strategy compare to this pipeline?

Importantly, we should note that both the joint and the pipeline strategies instantiate the declarative learning approach outlined in~\S \ref{sec:problem} where logical facts are compiled into an optimization learning problem via the logic relaxations. However, in the pipeline, the coherence constraints are not explicitly involved in the loss, because the noisy labels already satisfy them. Since conjunctions are identical for the $\mathcal{S}$ and $\mathcal{R}$ relaxations, they give the same results.

Tables~\ref{tab:supervised_digit},~\ref{tab:supervised_ope}, and~\ref{tab:supervised_coherence} show the results of these pipelined experiments.  We observe the same trends as in the joint learning experiments: $\mathcal{R}$-Product achieves the highest performance. 

Comparing the joint learning and the pipeline results, we observe that the latter are slightly better across relaxations and data settings. As mentioned above, however, pipelining is only viable here because of the special form of our constraint.

\begin{table}[t]
    \centering
    \begin{tabular}{rccc}
    \toprule
         & 1000 & 5000 & 25000 \\
    \midrule
     $\mathcal{S}$-\godel & 95.2 & 97.6 & 97.59 \\
     $\mathcal{R}/\mathcal{S}$-Product & \textbf{96.7} & \textbf{98.4} & \textbf{99.12} \\
     \luka & 96.5 & 98.5 & 98.76 \\
     \bottomrule
    \end{tabular}
    \caption{Pipelined \boolean{Digit} accuracies trained on DIGIT sizes 1k, 5k, 25k. Standard deviation is 0.0 for every entry.}
    \label{tab:supervised_digit}
\end{table}
\begin{table}[t]
    \centering
    \begin{tabular}{rccc}
    \toprule
         & 1000 & 5000 & 25000 \\
    \midrule
     $\mathcal{S}$-\godel & 88.3 {\tiny(0.3)} & 90.9 {\tiny(0.1)}  & 91.3 {\tiny(0.1)} \\
     $\mathcal{R}/\mathcal{S}$-Product & \textbf{89.7} {\tiny(0.1)} & \textbf{92.2} {\tiny(0.0)} & \textbf{93.0} {\tiny(0.0)}\\
     \luka & 83.0 {\tiny(1.9)} & 86.9 {\tiny(3.6)}  & 88.6 {\tiny(0.4)}\\
     \bottomrule
    \end{tabular}
    \caption{Average of Pipelined \boolean{Sum} and \boolean{Product} accuracies (and standard deviations) trained on PAIR size 5k (noisy) labeled with \texttt{Digit} model trained on DIGIT sizes 1k, 5k, 25k.}
    \label{tab:supervised_ope}
\end{table}
\begin{table}[t]
    \centering
    \begin{tabular}{rccc}
    \toprule
         & 1000 & 5000 & 25000 \\
    \midrule
     $\mathcal{S}$-\godel & 86.7 {\tiny(0.3)} & 90.6 {\tiny(0.1)} & 91.3 {\tiny(0.0)} \\
     $\mathcal{R}/\mathcal{S}$-Product & \textbf{89.9} {\tiny(0.2)} & \textbf{92.5} {\tiny(0.0)} & \textbf{92.8} {\tiny(0.0)} \\
     \luka & 83.5 {\tiny(1.8)} & 87.2 {\tiny(3.5)} & 88.7 {\tiny(0.4)}\\
     \bottomrule
    \end{tabular}
    \caption{Average of Pipelined \boolean{Sum} and \boolean{Prod} Coherence accuracies (and standard deviations) trained on PAIR size 5k (noisy) labeled with \boolean{Digit} model trained on DIGIT sizes 1k, 5k, 25k.}
    \label{tab:supervised_coherence}
\end{table}

\subsection{Text Chunking}
\label{subsec:chunking}
Our second set of experiments use the NLP task of text chunking using the CoNLL 2000 dataset~\citep{sang2000introduction}. This task illustrates how relaxed logic can be used to derive loss functions for sequential inputs and outputs.

Text chunking is a sequence tagging problem, where each word of a sentence is assigned a phrase type. For example, in the sentence `\textit{John is playing in the park.}', the word `\textit{John}' is labeled as \texttt{B-NP}, indicating that it starts a noun phrase (\texttt{NP}). We use the standard BIO labeling scheme: \texttt{B-X} indicates the start of a new phrase labeled \texttt{X}, \texttt{I-X} indicates the continuation of a phrase labeled \texttt{X}, and \texttt{O} marks words that do not belong to any of the predefined phrase types. 

In the case of text chunking, the input $x$ is a sequence of words and correspondingly, output $y$ is a sequence of labels (phrase types). Consequently, each position in the sequence is associated with a predicate. For an input $x$ with $n$ words, we have $n$ predicates, one per word.

Using our notation from~\S\ref{sec:problem} we define the predicate, $\pred{Tag}(x_i,y_i)$, to denote that $i^{th}$ word in input $x$ is assigned the label $y_i$. For our preceding example,  \textit{John} being assigned the label \texttt{B-NP} corresponds to the predicate $\pred{Tag}(\text{John, \texttt{B-NP}})$.

\subsubsection{Constraints} 
For each position in the sequence, we have constraints defining pairwise label dependencies. If a word in a sentence has a B/I label of a certain phrase type, then the next word cannot have a I label of a \emph{different} phrase type. For example, we have

\begin{tabular}{rp{0.76\linewidth}}
  $C_{i,1}$: & $\forall i, ~\pred{Tag}(x_i,\texttt{B-NP})\rightarrow \neg \pred{Tag}(x_{i+1},\texttt{I-VP})$  \\ 
  $C_{i,2}$: & $\forall i, ~\pred{Tag}(x_i,\texttt{I-NP})\rightarrow \neg \pred{Tag}(x_{i+1},\texttt{I-VP})$  \\
\end{tabular}

In general, given a labeled dataset $D$, and a set $C$ of $k$ constraints for each word position, we can write the conjunction of all constraints as:

\begin{align} 
\bigwedge_{x,y \in D}\pc{\bigwedge_{i} \p{\pred{Tag}(x_i,y_i) \bigwedge_{k} C_{i,k}}}
\end{align}

We can now use this Boolean formula to state the goal of learning as finding the model parameters of a neural network that maximizes the value of the relaxation derived for each t-norm. We used a bidirectional LSTM over GloVe embeddings~\cite{pennington2014glove} to instantiate \pred{Tag}.

\subsubsection{Experiments and Results}
We compare the four t-norms on two settings. First, in purely supervised learning, the model only learns from labeled examples. Second, we augment the labeled dataset with the constraints described above to study the effectiveness of incorporating simple constraints on outputs.  For both of these settings, we also study the models in a low data regime with training data restricted to 10\%. Following previous work~\citep{sang2000introduction}, we use the F1 score as our evaluation metric.

We report performances of the models trained with different t-norms in Table~\ref{tab:chunking_results}. We observe that for supervised learning (top rows), both variants of product t-norm outperform the other two t-norms by a significant margin. We again observe that for purely supervised learning, both $\mathcal{R}$-Product and $\mathcal{S}$-Product produce identical results since no implication constraints are used. Further, $\mathcal{S}$-\godel performs the worst among the four t-norms. Note that this observation is consistent with our analysis from Table~\ref{tab:tautotable} where $\mathcal{S}$-\godel was found to be least consistent with tautologies. 

Let us now look at the results of augmenting supervised learning with simple output constraints (Table~\ref{tab:chunking_results}, bottom rows). First, we observe that incorporating simple constraints into neural models, using the framework discussed in this work, can indeed improve their performance in all cases. Particularly, this improvement is more pronounced in low data regimes, since models need to rely more on external background knowledge when training data is small.

We also find that both variants of product t-norms outperform  $\mathcal{S}$-\godel and \luka. However, in both data regimes, $\mathcal{R}$-Product outperforms $\mathcal{S}$-Product when incorporating constraints into the neural model. We observe that these results are consistent with the ones obtained in~\S\ref{subsec:digits}.

\begin{table}[tp]
    \centering
    \setlength{\tabcolsep}{3.5pt}
    \begin{tabular}{rcccc}
    \toprule
     \% Train    & $\mathcal{S}$-\godel & $\mathcal{S}$-Prod. & $\mathcal{R}$-Prod. & \luka\\
    \midrule
     10\% & 70.29 & \textbf{79.46} & \textbf{79.46} & 76.50\\
     100\% & 85.33 & \textbf{89.18} & \textbf{89.18} & 87.25\\
     \midrule
     10\% + $C$ & 70.81  & 79.71 & \textbf{80.14} & 76.80\\
     100\% + $C$& 85.49 & 89.19 & \textbf{89.75} & 87.59 \\
     \bottomrule
    \end{tabular}
    \caption{Results for Text Chunking. F1 scores on test set of CoNLL 2000 dataset. Product t-norms outperform other t-norms for purely supervised setting (top rows). $\mathcal{R}$-Product performs best among all t-norms when constraints are augmented into the neural models (bottom rows). Rows with + $C$ include constraints.}
    \label{tab:chunking_results}
\end{table}

\subsection{Theory vs. Experiments}
\label{sec:theory-vs-experiments}

From our analysis of \tnorms in~\S\ref{sec:truthpres}, we found that \luka and $\mathcal{R}$-Product are the most preferable. Empirically, $\mathcal{R}$-Product outperforms the other relaxations.

The consistency analysis suggests that since  $\mathcal{S}$-\godel is least consistent, we should also expect $\mathcal{S}$-\godel to empirically perform the worst. This argument is reinforced by the results across all tasks and settings, if we naively applied the \godel \tnorm to define the loss. The $\mathcal{S}$-\godel results shown in this work were obtained by warm starting the learning using the $\mathcal{R}$-Product relaxation. Without the warm start, we found that not only is learning unstable, the resulting accuracies were also low.

One other discrepancy bears attention: although \luka is most preferable in terms of
consistency on tautologies, the experiments suggest otherwise. From an empirical perspective, a valid relaxation of logic should provide a sufficient gradient signal to make learning feasible. We discovered that loss functions defined by the \luka\tnorm are not amenable to gradient-based learning. Here, we briefly explain this phenomenon. 

Consider the \luka conjunction over $n$ atoms: 
$$\pb{\bigwedge_{i}^n \boolean{A}_i}=\max\p{0,\sum_i^n a_i - (n-1)}$$

For this operation to provide a non-zero output, we need $\sum_i^n a_i >  (n-1)$, or on average, each of the conjuncts $a_i$ should exceed $\frac{n-1}{n}$. 

That is, to provide a useful signal, \luka \tnorm requires the model to assign high probabilities to correct labels. This, of course, is not true for a randomly initialized model when learning starts, contributing to near-zero gradients, and no model updates at all.

To make learning feasible with \luka \tnorm,  we implemented a less strict definition of conjunction inspired by the MAX-SAT relaxation of \citet{bach2017hinge}.
This transforms our learning objective to the form 
 $  \max_\theta\p{\sum_i^n a_i}$.
An important consequence of this approach is that \luka \tnorm in our experiments is merely an approximation of the original \luka \tnorm.

\section{Related Work and Discussion}

The use of \tnorm relaxations to encode knowledge into learning is increasingly prevalent in recent years~\citep[for example]{wang2020joint,minervini2017adversarial,donadello2017LTN}. 
Additionally,~\citet{grefenstette2013distributional}, and~\citet{nandwani2019primal} also implicitly use \tnorms to similar ends. These works do not frame the learning problem as a declarative statement encoded using a single predefined logical language, as in the case of~\cite{sikka2020deep} and~\cite{DBLP:conf/ilp/GianniniMDMG19}.

Our framework is closer to~\citet{li2019logic-driven},~\citet{asai2020logicguided}, and~\citet{wang2020joint} where both data and background knowledge are declaratively stated and encoded in one single \tnorm relaxation that defines the loss.
Our experiments for the extended MNIST digit classification are inspired by~\citet{manheve2018deeproblog}, who employ constraints similar to our coherence constraints.
Of course, the goal of this work is to theoretically, and empirically investigate which relaxation is best suited for use in such problems. To our knowledge, none of the previously mentioned works analyze the choice of the relaxation they use. 

Similar studies to this paper are that of~\citet{EvansG18}, and~\citet{kriekenAH20fuzzy}. The former, includes comparison in performance of \luka, Product, and \godel \tnorm operators used to induce differentiable functions from definitive clauses in neural program synthesis. The latter, perhaps closer to our approach, provides a general theoretical and empirical analysis for different \tnorm relaxations. 
However, there are two key differences with our work: (a) they do not treat labeled data as part of the declarative problem specification, and the constraints are added into a standard cross entropy loss, (b) their  analysis involves the derivatives of the losses and does not discuss the consistency and self-consistency properties. We also studied the importance of the loss gradient signal in gradient-based learning, and hypothesise that a less consistent t-norm would perform worse at characterizing the truth of a Boolean statement than a more consistent one, and as a result, would correspond to poorer empirical performance.

\section{Conclusions}

In this work, we studied the question of how best to relax declarative knowledge to define loss functions. To this end, we define a set of criteria that characterizes which relaxations of logic would be most amenable for preserving tautologies and offer support for gradient-based learning. We also present empirical studies on two tasks using the paradigm of formulating entire learning problems via logic. All our analyzes concur that the $\mathcal{R}$-Product relaxation is best suited for learning in this paradigm. 


\section*{Acknowledgments}
 The authors acknowledge the support of NSF grants \#1801446 (SaTC) and \#1822877 (Cyberlearning). We also thank Ganesh Gopalakrishnan and the members of the Utah NLP group  for their feedback on previous iterations of this work, and the IJCAI reviewers for their valuable feedback.

\bibliographystyle{named}
\bibliography{ijcai21}

\appendix
\newpage 
\section{Relevant Proofs}

\textbf{Proof for Proposition \ref{prop:R-is-self-consist}.}

\begin{proof}

The \tnorms we consider (\godel, \luka and Product) are left-continuous binary operators. For a left-continuous \tnorm $T$, we can uniquely define the residuum (denoted by $\rightarrow$) as the binary operator that satisfies
\begin{align}
    \forall x,y,z \in[0,1],\quad T(z,x)\leq y \text{ iff. } z\leq (x\rightarrow y).
\end{align}
The residuum generalizes the logical implication. From the definition, we see that $x\rightarrow y$ is greater than or equal to every $z$ such that $T(z, x) \leq y$. This allows us to write the residuum in the following equivalent way:
\begin{equation}\label{eq:residua_prop}
    (x\rightarrow y) = \sup\{z\mbox{ }|\mbox{ } T(z,x)\leq y\}.
\end{equation}
Next, we use the monotonicity property of \tnorms. Every \tnorm is a non-decreasing function of both its arguments. In particular, we have:
\begin{align}
    T(z, x) \leq T(1, x) = x
\end{align}
Consequently, for any $x, y \in [0,1]$ such that $x \leq y$, we have $T(z, x) \leq y$. For such $x, y$ pairs, the supremum in \eqref{eq:residua_prop} is identical to $1$. Indeed, we can show that whenever the residuum $x\rightarrow y$ evaluates to 1, we have $x \leq y$. In short,
\begin{align}\label{eq:residuum_one}
    (x\rightarrow y)=1\text{ iff. }x\leq y.
\end{align}
Now, consider the self-consistency of a well-formed Boolean formula \boolean{P} and an $\mathcal{R}$-logic relaxation L, denoted by $\kappa_S^L(\text{\boolean{P}})$. From the definition of self-consistency, we have:
\begin{align*}
    \kappa_S^L\p{\text{\boolean{P}}} & = \kappa^L\p{\text{\boolean{P}}\leftrightarrow\text{\boolean{P}}} \\ 
    & =  \kappa^L\p{\p{\boolean{P} \leftarrow \boolean{P}}\wedge \p{\boolean{P} \rightarrow \boolean{P}}}
\end{align*}
When we relax the conjunction in the final expression, we will end up with a t-norm, both of whose arguments are one from \eqref{eq:residuum_one}. Consequently, the t-norm itself evaluates to one. Applying the definition of consistency (the integral of the constant $1$ over the domain) gives us the required result.
\end{proof}

\noindent \textbf{Proof Sketch for Proposition \ref{prop:selfconprod}.}

\begin{proof}
 Let $F=A\leftrightarrow A=\bigwedge_{i=1}^n\boolean{A}_i\leftrightarrow\bigwedge_{i=1}^n\boolean{A}_i$. 
By the definition of the $\mathcal{S}$-Product logic connectives we have, 
    
    $$[F]=\bigg(1+\bigg(\prod_{i=1}^n a_i\bigg)^2-\bigg(\prod_{i=1}^n a_i\bigg)\bigg)^2$$ 
and by the definition of self-consistency, we have 
    
    $$\kappa_S^{\mathcal{S}\text{-Product}}(A)=\int_0^1[F]\hspace{0.1cm}\mathrm{d}\mathcal{F}$$ 
    
Let us simplify the large product corresponding to the n-conjunction by introducing an auxiliary variable that includes all but the final element in the conjunction $A$. We will call this $n-1$-conjunction $b_{n-1}$. That is,
    $$b_{n-1}=\prod_{i=1}^{n-1} a_i$$
This allows us to write $[F]$ as
    $$[F]=(1+a_n^2b_{n-1}^2-a_n b_{n-1})^2.$$ 

Integrating out the $n^{th}$ variable $a_n$ by expanding out the polynomial, we obtain
    \begin{align*}
    \kappa_n &=\int_0^1\Big(1+a_n^2 b^2_{n-1}-a_n b_{n-1}\Big)^2 \,\mathrm{d}a_n\\
    &=1-b_{n-1}+b_{n-1}^2-\frac{1}{2}b_{n-1}^3+\frac{1}{5}b_{n-1}^4.
    \end{align*}
    Applying the same strategy as before, we can define
    $$b_{n-2}=\prod_{i=1}^{n-2} a_i$$
    then, 
    $$b_{n-1}=a_{n-1}b_{n-2}.$$
    Substituting $b_{n-1}$ in $\kappa_n$ and taking the integral over the $(n-1)^{th}$ variable we obtain,
    \begin{align*}
    \kappa_{n-1}&=\int_0^1 \kappa_n \,\mathrm{d}a_{n-1}\\
    &=\int_0^1 1-a_{n-1} b_{n-2}+a_{n-1}^2 b^2_{n-2}\\
    &\hspace{10 mm}-\frac{1}{2}a_{n-1}^3 b_{n-2}^3+\frac{1}{5}a_{n-1}^4 b_{n-2}^4\,\hspace{0.1cm}\mathrm{d}a_{n-1}\\ 
    &=1-\frac{2}{2^2}b_{n-2}+\frac{3}{3^2}b_{n-2}^2-\frac{2}{4^2}b_{n-2}^3+\frac{1}{5^2} b_{n-2}^4.
    \end{align*}
    Repeating this procedure, every subsequent integral will add one to the degree of the denominators in each term. This allows us to generalize the integral as
    $$\kappa_{n-k}=1-\frac{2}{2^k}b_{n-k}+\frac{3}{3^k}b^k_{n-k}-\frac{2}{4^k}b_{n-k}^3+\frac{1}{5^k}b_{n-k}^4$$
    where,
    $$b_{n-k}=\prod_{i=1}^{n-k}a_i.$$
    In particular, when $k=n$, the empty product becomes 1. Therefore, substituting $b_{n-k}$ in $\kappa_{n-k}$ we obtain,
    
    $$\kappa_S^{\mathcal{S}\text{-Product}}(A)=\kappa_0=1-\frac{2}{2^n}+\frac{3}{3^n}-\frac{2}{4^n}+\frac{1}{5^n}$$
    
    \end{proof}

\section{$\mathcal{R}$-logics and the Connective for Negation}

In table \ref{tab:tnorms_definition}, we abuse the nomenclature for $\mathcal{R}$-Product and $\mathcal{R}$-\godel logics. In these relaxations, the negation and conjunction connectives are respectively defined as 
$$n_\neg(x)=\begin{cases}
       1 &$\quad\text{if  }x=0$\\
       0 &$\quad\text{otherwise}$
     \end{cases}$$ 
     
and, 

$$x\wedge y=\begin{cases}
       0 &$\quad\text{if  }x=y=0$\\
       1 &$\quad\text{otherwise}$
     \end{cases}$$

Unfortunately, these functions are not sub-differentiable. Instead, we introduce the corresponding SBL$_\sim$ extensions for residuated $\mathcal{R}$ \tnorm logics   ~\cite{esteva2000residuatedinvolutive} which are constructed using the involutive negation $n_\neg(x)=1-x$ and sub-differentiable in all of their connectives. We keep the same names $\mathcal{R}$-Product and $\mathcal{R}$-\godel for simplicity.

\section{Reproducibility}

\subsection{Digit Arithmetic experiments}

\paragraph{Model Architecture.} For \boolean{Digit}, \boolean{Sum} and \boolean{Product} models we use 2-layer Convolutional Neural Networks with 3-by-3 padded filters, with ReLU activation and 2-by-2 Maxpool layers with stride 2 in between, with 2 fully connected layers with ReLU activation and dropout ($p=0.5$) in between. 

\paragraph{Seeds and multiple runs. } We use a random seed 20 to create all the dataset used for training, evaluation, and testing. The same seed (20) is used for hyper-parameter tuning. We perform three runs of every experiment using seeds 0, 20, and 50. The standard deviation from the three runs with different seeds was high (up to 5\%) for some of the relaxations over the scarce data settings (DIGIT size 1000 and/or PAIR size 1000). In these cases, we perform extra runs with seeds 1, 10, and 60, reporting the average of the best three.

\paragraph{Hyper-parameter tuning Details}

Validation sets are used for tuning hyper-parameters: learning rate, batch size, optimizer, lambda coefficient for the coherence constraints in the loss. Specifically we perform grid search with: 
\begin{itemize}
    \item learning rates: \{$10^{-1}$, $5\times 10^{-2}$, $10^{-2}$, $5\times10^{-3}$, $10^{-3}$, $5\times10^{-4}$, $10^{-4}$, $5\times10^{-5}$, $10^{-5}$\}
    \item $\lambda$'s: $\{0.05, 0.1, 0.5, 1, 1.5, 2\}$
    \item batch sizes: \{8, 16, 32, 64 \}.
\end{itemize}
 We also treat the optimization method as a hyper-parameter, using the one resulting in best performance between standard Stochastic Gradient Descent and Adam optimizers. 

\paragraph{Number of epochs for training.} The number of epochs for convergence varies between 100 to 1000 epochs across \tnorm relaxations and data settings. All the experiments were run for a maximum of 1600 epochs.

\paragraph{Significance.}  We  performed  t-tests  and observed a statistical significance of at least $p<0.05$ in all the model comparisons we discuss in the main paper.

\paragraph{Warm-restarts.} Using the Stochastic Gradient Descent with Restarts techniques~\cite{loshchilov2016sgdr}, brings, considerably, more stability to the experiments. We use the best number iterations for the first  restart ($T_0$) among 50, 100 and 200,  with a factor of increase of 1 ($T_{mult}$)

\subsection{Experimental Details for Text Chunking}
\paragraph{Model Architecture. } For all our models, we use bidirectional RNN based models, experimenting with LSTMs and GRUs. We use GloVe word embeddings of 300 dimension at input along with 100 dimensional character embeddings. We fix the dropout rate to 0.5 for all our experiments.

\paragraph{Hyper-parameter Tuning}
Since the original CoNLL 2000 dataset does not accompany a validation set, for each experiment, we select 10\% of the total training data for validation. We perform hyper-parameter tuning using grid search on the following parameters: hidden unit size of RNN in the set \{200. 256, 300, 400, 500, 512\}, learning rate in the set \{0.001, 0.005, 0.0001\}. 
\subsection{Weighted Constrained Loss}
We observe that in order to stabilize training, we have to use a $\lambda$ (hyper-parameter) to relatively weigh the loss from both \boolean{Sum} and \boolean{Product} coherence constraints. We find that product based t-norms are less sensitive to different values of $\lambda$.

For chunking, for example, for \luka \tnorm, values around 0.0001 works best, while for $\mathcal{R}$-Product, values greater than 1 provides best results on validation set.

\subsection{Code and Computing Infrastructure}
 We implement all our experiments in the deep learning library PyTorch~\cite{paszke2019pytorch} on a server with the following configuration:
 \begin{itemize}
     \item CPU: AMD EPYC 7601, 32 cores, 2.2 GHz , 
     \item Memory: 512 gb, 
     \item GPU: Titan RTX, 
     \item Disk Space: 8 TB. 
 \end{itemize}

\section{Learning behaviour of neural models}

\subsection{Product relaxations}
    $\mathcal{S}$-Product and $\mathcal{R}$-Product logic are the most stable and less brittle \tnorm logics across different data settings.

\subsection{Discussion about Gradients in \luka relaxation}
\label{sec:non-zero}

From an empirical perspective, a valid relaxation of logic should provide enough gradient signal to make the gradient-based learning process possible. One relevant case of a logic relaxation that does not satisfy this property is the \luka \tnorm. By inspecting the \luka conjunction connective definition: 

\begin{equation}\pb{\bigwedge_{i}^n \boolean{A}_i}=\max\p{0,\sum_i^n a_i - (n-1)}
\label{eq:original_luka}
\end{equation}
we can see that it requires almost absolute certainty for each of its conjuncts to result in a non-zero truth value. This means that the gradient signal provided from an objective that is the conjunction of the dataset term (\eqref{eq:dataset}) and the constraint term (\eqref{eq:generic-constraint}) is almost always zero.




To be able to study the \luka logic, for our experiments we implemented a less strict definition of the conjunction inspired by the MAX-SAT relaxation of \citet{bach2017hinge} assuming the best case scenario $\sum_i^n a_i \geq n-1$. Therefore, the learning problem using \luka logic,  defined in~\S\ref{learning_logic} as

$$\max_\theta\p{\max\p{0,\sum_i^n a_i - (n-1)}}$$

is restated for our experiments as
\begin{equation}\label{eq:relaxd_luka}
  \max_\theta\p{\sum_i^n a_i}
\end{equation}




\subsection{Practical considerations for $\mathcal{S}$\text{-\godel} logic}
In practice, we find that optimizing the loss function derived from the $\mathcal{S}$-\godel relaxation was difficult due to learning instability. To make learning possible, we apply the following techniques:
\begin{itemize}
        \item By definition of $\mathcal{S}$-\godel conjunction, for every epoch, we make the update with respect to the example with the minimum output value in the data. Instead, we apply mini-batch descent over the minimum of each batch, and proceed with the standard SGD process. This means that we use the $\mathcal{S}$-\godel logic over the batches but the $\mathcal{S}$-Product \tnorm (which agrees with the cross-entropy loss) over the epochs. While this is not strictly a \godel conjunction, we found that this was necessary for the optimizer to work.
        \item Given that at the early stages of learning the system is uncertain about every example, and $\mathcal{S}$-\godel conjunction operates over the ``worst'' example, it is hard to get any meaningful signal from the updates. To break ties, we ``warm-up'' the system by running 2 to 10 epochs of learning using the $\mathcal{R}$-Product logic. Interestingly, even though this strategy is needed to learn with $\mathcal{S}$-\godel on both joint and pipeline learning frameworks, the joint system needs only one or two epochs to warm-up, while the pipelined approach needs between 6 to 10 epochs to start learning on its own. These warm-up processes represent on average 40\% of the final accuracies obtained across our experiments.  
        \item In the joint model, we noticed that the accuracies for the \boolean{Digit} classifier decrease after a few epochs hindering training. This is caused because the model is trying to reach the trivial solution to satisfy the coherence constraints (the \boolean{Digit} classifier predicting low scores for all digits to make the value of the conjunction in the antecedent of the constraints as low as possible). In order to avoid this, we froze the parameters corresponding to the \boolean{Digit} model right after the warm-up stage; in other words, we train \boolean{Digit} classifier using $\mathcal{R}$-Product relaxation for a few epochs; enough to successfully continue the training. 
\end{itemize}

\section{Examples of Consistency Calculation} 

Here we calculate consistency under different relaxations for the axiom schema tautology $\boolean{P}\rightarrow (\boolean{Q}\rightarrow\boolean{P})$ which we will denote by $A1$.\\

\noindent \begin{itemize}
\item $\mathcal{S}$-\godel \tnorm.
\begin{align*}
    [A1]&= \max\Big(1-[\boolean{P}], [\boolean{Q}\rightarrow\boolean{P}]\Big)\\
    &=\max\Big(1-p,\max(1-q, p)\Big)\\ 
\end{align*}
The consistency is given by
\begin{align*}
    & \kappa^{\mathcal{S}\text{-\godel}}(A1) \\ 
     &\quad=\int_0^1\int_0^1 \max\Big(1-p,\max(1-q, p)\Big)\hspace{.1cm} \mathrm{d}p\hspace{0.05cm}\mathrm{d}q \\
     &\quad=\frac{19}{24}\approx 0.79 
\end{align*}

\item $\mathcal{S}$-Product \tnorm.
\begin{align*}
    [A1]&= 1-[\boolean{P}]+([\boolean{P}]\cdot[\boolean{Q}\rightarrow\boolean{P}])\\
    &=1-p+(p\cdot(1-q+q\cdot p))\\ 
\end{align*}

The consistency is given by
\begin{align*}
    & \kappa^{\mathcal{S}\text{-Product}}(A1) \\ 
     &\quad=\int_0^1\int_0^1 1-p+(p\cdot(1-q+q\cdot p))\hspace{.1cm} \mathrm{d}p\hspace{0.05cm}\mathrm{d}q \\
     &\quad=\frac{11}{12}\approx0.92 
\end{align*}

\item $\mathcal{R}$-Product \tnorm.

$$[A1]=\left\{
        \begin{array}{ll}
            1 & \text{if}\quad [\boolean{P}]\leq[\boolean{Q}\rightarrow\boolean{P}] \\
            \frac{[\boolean{Q}\rightarrow\boolean{P}]}{[\boolean{P}]} & \quad\text{else}
        \end{array}
    \right.\\$$

We have two cases:
\begin{enumerate}
\item $[\boolean{P}]< [\boolean{Q}]$ \\

By the definition of $\mathcal{R}$-Product implication, 

$$[\boolean{\boolean{Q}}\rightarrow\boolean{P}]=\frac{p}{q}$$
then, 
$$[A1]=\left\{
        \begin{array}{ll}
            1 & \text{if}\quad p\leq\frac{p}{q} \\
            \frac{1}{q} & \quad\text{else}
        \end{array}
    \right.\\$$
but notice that $p\not>\frac{p}{q}$ because $0<q<1$ by definition. Therefore, $[A1]=1$.

\item $[\boolean{P}]\geq[\boolean{Q}]$ \\

By the definition of $\mathcal{R}$-Product implication,

$$[\boolean{\boolean{Q}}\rightarrow\boolean{P}]=1$$
then, by the definition of $[A1]$ under $\mathcal{R}$-Product and the fact that $0<[\boolean{P}]<1$ we obtain that $[A1]=1$.
\end{enumerate}
\vspace{0.3cm}
\item\luka \tnorm.

\begin{align*}
    [A1]&= \min\Big(1, 1-[\boolean{P}]+\min\big(1, 1-[\boolean{Q}]+[\boolean{P}]\big)\Big)\\
    &=\min\Big(1, 1-p+\min(1-q+p)\Big)\\ 
\end{align*}

The consistency is given by
\begin{align*}
    & \kappa^{\mathcal{S}\text{-\luka}}(A1) \\ 
     &\quad=\int_0^1\int_0^1 \min\Big(1, 1-p+\min(1-q+p)\Big)\hspace{.1cm} \,\mathrm{d}p\hspace{0.05cm}\,\mathrm{d}q \\
     &\quad=1 
\end{align*}

\end{itemize}

\section{Additional Experiments and Results}

\subsection{Recognizing Digits and Arithmetic Operations}

For the same set of experiments described in~\cref{subsec:digits}, tables~\ref{tab:joint_digit_1000},~\ref{tab:joint_ope_1000},~\ref{tab:joint_coherence_1000},~\ref{tab:joint_digit_25000},~\ref{tab:joint_ope_25000},~\ref{tab:joint_coherence_25000} report the results from instantiating the sum and product coherence constraints over 1k and 25k PAIR examples,respectively, to obtain the training signal for the \boolean{Sum} and \boolean{Prod} classifiers.

\begin{table}[htb!]
    \centering
    \begin{tabular}{rccc}
    \toprule
         & 1000 & 5000 & 25000 \\
    \midrule
     $\mathcal{S}$-\godel & 94.5 {\tiny (0.3)} & 97.0 {\tiny (0.2)} & 96.7 {\tiny (0.1)} \\
     $\mathcal{S}$-Product & 92.5 {\tiny (0.1)} & 97.8 {\tiny (0.0)} & \textbf{99.0} {\tiny (0.1)}\\
     $\mathcal{R}$-Product & \textbf{95.5} {\tiny (0.1)} & \textbf{97.9} {\tiny (0.0)} & 96.0  {\tiny (0.0)}\\
     \luka & 95.1 {\tiny (0.4)} & 96.2 {\tiny (0.2)} & 97.9 {\tiny (0.3)} \\
     \bottomrule
    \end{tabular}
    \caption{\boolean{Digit} accuracies (and standard deviations) from jointly training \boolean{Digit} on DIGIT sizes 1k, 5k, 25k and operators (\boolean{Sum},\boolean{Prod}) on PAIR size 1k.}
    \label{tab:joint_digit_1000}
\end{table}

\begin{table}[htb!]
    \centering
    \begin{tabular}{rccc}
    \toprule
         & 1000 & 5000 & 25000 \\
    \midrule
     $\mathcal{S}$-\godel & \textbf{62.7} {\tiny(0.9)} & \textbf{62.0} {\tiny(0.3)} & 64.1 {\tiny(0.2)} \\
     $\mathcal{S}$-Product & 49.6 {\tiny(1.1)} & 59.1 {\tiny(1.0)} & 60.6 {\tiny(0.3)} \\
     $\mathcal{R}$-Product & 53.3 {\tiny(0.7)} & 59.3 {\tiny(0.5)} & \textbf{66.5} {\tiny(0.2)}\\
     \luka & 44.5 {\tiny(4.2)} & 45.3 {\tiny(3.1)} & 50.8 {\tiny(0.7)}\\
     \bottomrule
    \end{tabular}
    \caption{Average of \boolean{Sum} and \boolean{Product} accuracies (and standard deviations) from jointly training \boolean{Digit} on DIGIT sizes 1k, 5k, 25k and operators on PAIR size 1k.}
    \label{tab:joint_ope_1000}
\end{table}
\begin{table}[htb!]
    \centering
    \begin{tabular}{rccc}
    \toprule
         & 1000 & 5000 & 25000 \\
    \midrule
     $\mathcal{S}$-\godel & \textbf{64.1} {\tiny(0.9)} & \textbf{62.1} {\tiny(0.2)} & 63.9 {\tiny(0.2)}\\
     $\mathcal{S}$-Product & 52.1 {\tiny(1.3)} & 59.4 {\tiny(1.2)} & 60.7 {\tiny(0.5)}\\
     $\mathcal{R}$-Product & 54.4 {\tiny(0.5)} & 59.1 {\tiny(0.2)} & \textbf{67.0} {\tiny(0.0)} \\
     \luka & 45.4 {\tiny(4.5)} & 45.8 {\tiny(2.9)} & 50.9 {\tiny(1.0)}\\
     \bottomrule
    \end{tabular}
    \caption{Average of \boolean{Sum} and \boolean{Prod} Coherence accuracies (and standard deviations) from jointly training \boolean{Digit} on DIGIT sizes 1k, 5k, 25k and operators on PAIR size 1k.}
    \label{tab:joint_coherence_1000}
\end{table}

\begin{table}[htb!]
    \centering
    \begin{tabular}{rccc}
    \toprule
         & 1000 & 5000 & 25000 \\
    \midrule
     $\mathcal{S}$-\godel & 93.5 {\tiny(0.2)} & 97.7 {\tiny(0.1)} & 98.0 {\tiny(0.0)}\\
     $\mathcal{S}$-Product & 95.2 {\tiny(0.0)} & 98.2 {\tiny(0.0)}& 99.0 {\tiny(0.0)}\\
     $\mathcal{R}$-Product & \textbf{96.2} {\tiny(0.0)} & \textbf{98.4} {\tiny(0.0)} & 99.0 {\tiny(0.0)}\\
     \luka & 94.8 {\tiny(0.1)} & 98.2 {\tiny(0.1)} & \textbf{99.0} {\tiny(0.0)}\\
     \bottomrule
    \end{tabular}
    \caption{\boolean{Digit} accuracies (and standard deviations) from jointly training \boolean{Digit} on DIGIT sizes 1k, 5k, 25k and operators (\boolean{Sum},\boolean{Prod}) on PAIR size 25k.}
    \label{tab:joint_digit_25000}
\end{table}
\begin{table}[htb!]
    \centering
    \begin{tabular}{rccc}
    \toprule
         & 1000 & 5000 & 25000 \\
    \midrule
     $\mathcal{S}$-\godel & 88.3 {\tiny (0.5)} & 94.9 {\tiny (0.2)} & 95.9 {\tiny (0.0)}\\
     $\mathcal{S}$-Product & 83.6 {\tiny (0.5)} & 90.9 {\tiny (0.5)} & 95.0 {\tiny (0.0)}\\
     $\mathcal{R}$-Product & \textbf{89.6} {\tiny (0.2)} & \textbf{95.7} {\tiny (0.1)} & \textbf{96.1} {\tiny (0.1)} \\
     \luka & 70.8 {\tiny (3.0)} & 89.4 {\tiny (1.4)} & 90.4 {\tiny (0.4)} \\
     \bottomrule
    \end{tabular}
    \caption{Average of \boolean{Sum} and \boolean{Product} accuracies (and standard deviations) from jointly training \boolean{Digit} on DIGIT sizes 1k, 5k, 25k and operators on PAIR size 25k.}
    \label{tab:joint_ope_25000}
\end{table}
\begin{table}[htb!]
    \centering
    \begin{tabular}{rccc}
    \toprule
         & 1000 & 5000 & 25000 \\
    \midrule
     $\mathcal{S}$-\godel & 86.0 {\tiny(0.4)} & 94.6 {\tiny(0.2)} & 95.4 {\tiny(0.1)} \\
     $\mathcal{S}$-Product & 85.8 {\tiny(0.5)} & 91.7 {\tiny(0.4)} & 95.6 {\tiny(0.1)} \\
     $\mathcal{R}$-Product & \textbf{90.8} {\tiny(0.1)} & \textbf{95.9} {\tiny(0.0)}  & \textbf{96.3} {\tiny(0.0)}\\
     \luka & 72.9 {\tiny(3.1)} & 90.2 {\tiny(1.2)} & 90.7 {\tiny(0.4)} \\
     \bottomrule
    \end{tabular}
    \caption{Average of \boolean{Sum} and \boolean{Prod} Coherence accuracies (and standard deviations) from jointly training \boolean{Digit} on DIGIT sizes 1k, 5k, 25k and operators on PAIR size 25k.}
    \label{tab:joint_coherence_25000}
\end{table}

Similarly, tables~\ref{tab:supervised_ope_1000},~\ref{tab:supervised_coherence_1000},~\ref{tab:supervised_ope_25000},~\ref{tab:supervised_coherence_25000} report the results from assigning noisy labels to 1k and 25k unlabeled PAIR data examples to train the \boolean{Sum} and \boolean{Prod} classifiers independently following the pipeline strategy we use in~\cref{subsec:digits}.

\begin{table}[htb!]
    \centering
    \begin{tabular}{rccc}
    \toprule
         & 1000 & 5000 & 25000 \\
    \midrule
     $\mathcal{S}$-\godel & 60.7 {\tiny (0.6)} & \textbf{63.0}  {\tiny (0.4)} & 62.9  {\tiny (0.2)}\\
     $\mathcal{R}/\mathcal{S}$-Product & \textbf{61.5} {\tiny (0.2)} & 60.9 {\tiny (0.3)} & \textbf{62.5} {\tiny (0.1)}\\
     \luka & 54.9 {\tiny (3.1)} & 55.5 {\tiny (6.7)} & 55.1 {\tiny (0.3)}\\
     \bottomrule
    \end{tabular}
    \caption{Average of Pipelined \boolean{Sum} and \boolean{Product} accuracies (and standard deviations) trained on PAIR size 1k (noisy) labeled with \texttt{Digit} model trained on DIGIT sizes 1k, 5k, 25k.}.
    \label{tab:supervised_ope_1000}
\end{table}
\begin{table}[htb!]
    \centering
    \begin{tabular}{rccc}
    \toprule
         & 1000 & 5000 & 25000 \\
    \midrule
     $\mathcal{S}$-\godel & 60.7 {\tiny (0.4)} & \textbf{63.5} {\tiny (0.3)} & 62.3 {\tiny (0.4)} \\
     $\mathcal{R}/\mathcal{S}$-Product & \textbf{62.1} {\tiny (0.4)} & 61.2 {\tiny (0.3)} & \textbf{62.6} {\tiny (0.3)}\\
     \luka & 55.6 {\tiny (2.4)} & 55.7 {\tiny (3.4)} & 55.6 {\tiny (1.0)} \\
     \bottomrule
    \end{tabular}
    \caption{Average of Pipelined \boolean{Sum} and \boolean{Prod} Coherence accuracies (and standard deviations) trained on PAIR size 1k (noisy) labeled with \boolean{Digit} model trained on DIGIT sizes 1k, 5k, 25k.}
    \label{tab:supervised_coherence_1000}
\end{table}
\begin{table}[htb!]
    \centering
    \begin{tabular}{rccc}
    \toprule
         & 1000 & 5000 & 25000 \\
    \midrule
     $\mathcal{S}$-\godel & 90.5 {\tiny (0.1)} & 94.5 {\tiny (0.0)} & 95.4 {\tiny (0.0)}\\
     $\mathcal{R}/\mathcal{S}$-Product & \textbf{91.1} {\tiny (0.2)} & \textbf{95.3} {\tiny (0.1)} & \textbf{96.9} {\tiny (0.0)}\\
     \luka & 85.9 {\tiny (1.4)} & 93.3 {\tiny (1.9)} & 95.4 {\tiny (0.3)}\\
     \bottomrule
    \end{tabular}
    \caption{Average of Pipelined \boolean{Sum} and \boolean{Product} accuracies (and standard deviations) trained on PAIR size 25k (noisy) labeled with \texttt{Digit} model trained on DIGIT sizes 1k, 5k, 25k.}.
    \label{tab:supervised_ope_25000}
\end{table}
\begin{table}[htb!]
    \centering
    \begin{tabular}{rccc}
    \toprule
         & 1000 & 5000 & 25000 \\
    \midrule
     $\mathcal{S}$-\godel & 88.5 {\tiny (0.2)} & 93.8 {\tiny (0.2)} & 94.8 {\tiny (0.0)}\\
     $\mathcal{R}/\mathcal{S}$-Product & \textbf{91.0} {\tiny (0.1)} & \textbf{95.3} {\tiny (0.1)} & \textbf{96.9} {\tiny (0.0)}\\
     \luka & 85.8 {\tiny (1.5)} & 93.3 {\tiny (1.5)} & 95.4 {\tiny (0.3)} \\
     \bottomrule
    \end{tabular}
    \caption{Average of Pipelined \boolean{Sum} and \boolean{Prod} Coherence accuracies (and standard deviations) trained on PAIR size 25k (noisy) labeled with \boolean{Digit} model trained on DIGIT sizes 1k, 5k, 25k.}
    \label{tab:supervised_coherence_25000}
\end{table}

For all these different data configurations, we observed a similar trend in the results to the one observed in~\cref{subsec:digits} with $\mathcal{R}$-Product achieving the highest performance among the different relaxations.

\subsubsection{Arithmetic Properties}

We also evaluate the resulting models from both of the training regimes (joint and pipelines) on arithmetic properties they are not being trained for. Specifically, we create test sets to evaluate if the commutativity, associativity, and distributivity properties are satisfied by the operator classifiers.

To evaluate commutativity, we check how many times an operator classifier predicts the same digit for pairs of TEST images and their reverse order. 

To evaluate the associativity property 
$$\forall
x_1,x_2,x_3 \quad (x_1+x_2)+x_3=x_1+(x_2+x_3)$$
in a operator classifier, say the \boolean{Sum} (mod 10) operator, we use triples of different images from TEST. For each triple $x_1$, $x_2$, $x_3$, we consider the predicted digits $y_{1,2}$ and $y_{2,3}$ corresponding to the \boolean{Sum} of $x_1$ and $x_2$ ($x_1+x_2$), and the \boolean{Sum} of $x_2$ and $x_3$ ($x_2+x_3$) respectively. Then, we randomly sample an image $x_{1,2}$ of a $y_{1,2}$, and an image $x_{2,3}$ of a $y_{2,3}$ from TEST, and check whether the respective predicted \boolean{Sum} of $x_{1,2}$ and $x_3$ ($x_{1,2}+x_3$) and \boolean{Sum} of $x_1$ and $x_{2,3}$ ($x_1+x_{2,3}$) agree. To compensate to the random choice of the image in TEST representing the predicted associations in the property, we repeat the process for 6 iterations and consider the most frequent resulting digit on each side of the equation respectively to make the final comparison. 

We follow an analogous approach to evaluate the distributivity property which involves predictions from both \boolean{Sum} and \boolean{Prod} classifiers.

In tables~\ref{tab:joint_properties_1000},~\ref{tab:joint_properties_5000},~\ref{tab:joint_properties_25000},~\ref{tab:supervised_properties_1000},~\ref{tab:supervised_properties_5k},~\ref{tab:supervised_properties_25k}, the first two columns respectively show the average for \boolean{Sum} and \boolean{Prod} classifiers of the fraction of test instances where the commutativity and associativity properties are satisfied. Similarly, the last column shows the average of the fraction of examples satisfying the left and right distributivity properties.
In these experiments, we observe that \godel and $\mathcal{R}$-Product are the best relaxations. On the other hand, \luka \tnorm performance was considerably poor with high standard deviations.

\begin{table}[htb!]
    \centering
    \begin{tabular}{rrccc}
    \toprule
         & & Commut. & Assoc. & Dist.  \\
         \midrule
         1000  &  $\mathcal{S}$-\godel & \textbf{59.1} {\tiny(3.4)} & \textbf{47.6} {\tiny(2.4)} & \textbf{43.3} {\tiny(4.1)}\\
         & $\mathcal{S}$-Product & 55.8 {\tiny(1.4)} & 45.6 {\tiny(1.9)} & 39.7 {\tiny(2.5)} \\
         & $\mathcal{R}$-Product & 49.4 {\tiny(1.8)} & 41.4 {\tiny(1.2)} & 32.1 {\tiny(1.9)} \\
         & \luka & 43.0 {\tiny(5.4)} & 37.6 {\tiny(5.2)} & 27.2 {\tiny(4.8)}\\
         \midrule
        5000 & $\mathcal{S}$-\godel & \textbf{60.0} {\tiny(2.6)} & 48.1 {\tiny(2.0)} & \textbf{41.8} {\tiny(2.9)}\\
        & $\mathcal{S}$-Product & 54.2 {\tiny(1.2)} & 43.5 {\tiny(1.2)} & 37.2 {\tiny(1.9)} \\
        & $\mathcal{R}$-Product & 54.0 {\tiny(1.1)} & \textbf{48.9} {\tiny(1.0)} & 41.4 {\tiny(1.2)} \\
        & \luka & 43.8 {\tiny(6.0)} & 39.9 {\tiny(5.7)} & 31.1 {\tiny(3.7)} \\
        \midrule
        25000 & $\mathcal{S}$-\godel & 62.7 {\tiny(0.5)} & 50.0 {\tiny(1.0)} & 42.7 {\tiny(0.9)} \\
        & $\mathcal{S}$-Product & 52.8 {\tiny(0.4)} & 44.7 {\tiny(0.3)} & 39.0 {\tiny(0.4)} \\
        & $\mathcal{R}$-Product & \textbf{63.0} {\tiny(0.2)} & \textbf{51.8} {\tiny(0.3)} & \textbf{44.7} {\tiny(0.3)} \\
        & \luka & 52.2 {\tiny(4.2)} & 44.9 {\tiny(4.6)} & 36.8 {\tiny(3.6)} \\
         \bottomrule
    \end{tabular}
    \caption{Arithmetic properties accuracies (and standard deviations) from jointly training \boolean{Digit} on DIGIT sizes 1k, 5k, 25k and operators (\boolean{Sum},\boolean{Prod}) on PAIR size 1k.}
    \label{tab:joint_properties_1000}
\end{table}
\begin{table}[htb!]
    \centering
    \begin{tabular}{rrccc}
    \toprule
         & & Commut. & Assoc. & Dist.  \\
         \midrule
         1000  &  $\mathcal{S}$-\godel & 90.0 {\tiny(0.3)} & 85.8 {\tiny(0.3)} & \textbf{82.6} {\tiny(0.3)}\\
         & $\mathcal{S}$-Product & 79.5 {\tiny(1.0)} & 69.7 {\tiny(0.7)} & 66.5 {\tiny(1.2)} \\
         & $\mathcal{R}$-Product & \textbf{91.0} {\tiny(0.8)} & \textbf{87.0} {\tiny(0.9)} & 82.0 {\tiny(0.5)} \\
         & \luka & 76.5 {\tiny(2.6)} & 65.7 {\tiny(4.1)} & 64.3 {\tiny(5.7)} \\
         \midrule
        5000 & $\mathcal{S}$-\godel & \textbf{91.8} {\tiny(0.0)} & \textbf{89.2} {\tiny(0.2)} & \textbf{86.8} {\tiny(0.1)} \\
        & $\mathcal{S}$-Product & 89.4 {\tiny(0.5)} & 85.6 {\tiny(0.4)} & 83.0 {\tiny(0.5)}\\
        & $\mathcal{R}$-Product & 90.9 {\tiny(0.5)} & 88.4 {\tiny(0.6)} & 83.9 {\tiny(0.1)} \\
        & \luka & 85.3 {\tiny(1.8)} & 80.2 {\tiny(2.9)} & 74.7 {\tiny(4.5)} \\
        \midrule
        25000 & $\mathcal{S}$-\godel & \textbf{91.9} {\tiny(0.0)} & \textbf{89.4} {\tiny(0.1)} & \textbf{87.2} {\tiny(0.0)} \\
        & $\mathcal{S}$-Product & 90.1 {\tiny(0.1)} & 86.8 {\tiny(0.0)} & 84.3 {\tiny(0.1)} \\
        & $\mathcal{R}$-Product & 91.4 {\tiny(0.1)} & 89.3 {\tiny(0.2)} & 84.1 {\tiny(0.0)} \\
        & \luka & 86.2 {\tiny(0.6)} & 81.2 {\tiny(0.7)} & 67.7 {\tiny(4.0)} \\
         \bottomrule
    \end{tabular}
    \caption{Arithmetic properties accuracies (and standard deviations) from jointly training \boolean{Digit} on DIGIT sizes 1k, 5k, 25k and operators (\boolean{Sum},\boolean{Prod}) on PAIR size 5k.}
    \label{tab:joint_properties_5000}
\end{table}
\begin{table}[htb!]
    \centering
    \begin{tabular}{rrccc}
    \toprule
         & & Commut. & Assoc. & Dist.  \\
         \midrule
         1000  &  $\mathcal{S}$-\godel & 93.9 {\tiny(0.1)} & 89.5 {\tiny(0.1)} & 85.7 {\tiny(0.0)}\\
         & $\mathcal{S}$-Product & 89.8 {\tiny(0.4)} & 82.2 {\tiny(0.2)} & 79.0 {\tiny(0.6)} \\
         & $\mathcal{R}$-Product & \textbf{95.3} {\tiny(0.3)} & \textbf{91.8} {\tiny(0.4)} & \textbf{86.3} {\tiny(0.3)} \\
         & \luka & 79.9 {\tiny(2.5)} & 63.3 {\tiny(4.9)} & 55.7 {\tiny(6.9)} \\
         \midrule
        5000 & $\mathcal{S}$-\godel & 96.6 {\tiny(0.0)} & 95.0 {\tiny(0.0)} & 92.2 {\tiny(0.0)} \\
        & $\mathcal{S}$-Product & 92.3 {\tiny(0.1)} & 88.8 {\tiny(0.2)} & 86.5 {\tiny(0.2)} \\
        & $\mathcal{R}$-Product & \textbf{97.5} {\tiny(0.2)} & \textbf{96.4} {\tiny(0.0)} & \textbf{93.8} {\tiny(0.0)}\\
        & \luka & 95.5 {\tiny(0.3)} & 93.7 {\tiny(0.3)} & 83.6 {\tiny(0.4)} \\
        \midrule
        25000 & $\mathcal{S}$-\godel & 96.7 {\tiny(0.0)} & 95.6 {\tiny(0.0)} & \textbf{94.5} {\tiny(0.0)} \\
        & $\mathcal{S}$-Product & 96.4 {\tiny(0.2)} & 94.9 {\tiny(0.1)} & 93.6 {\tiny(0.2)} \\
        & $\mathcal{R}$-Product & \textbf{97.5} {\tiny(0.0)} & \textbf{96.5} {\tiny(0.2)} & 94.1 {\tiny(0.1)}\\
        & \luka & 94.9 {\tiny(0.3)} & 91.5 {\tiny(0.1)} & 83.2 {\tiny(0.3)} \\
         \bottomrule
    \end{tabular}
    \caption{Arithmetic properties accuracies (and standard deviations) from jointly training \boolean{Digit} on DIGIT sizes 1k, 5k, 25k and operators (\boolean{Sum},\boolean{Prod}) on PAIR size 25k.}
    \label{tab:joint_properties_25000}
\end{table}
\begin{table}[htb!]
    \centering
    \begin{tabular}{rrccc}
    \toprule
         & & Commut. & Assoc. & Dist.  \\
         \midrule
         1000  &  $\mathcal{S}$-\godel & 56.4 {\tiny (0.9)} & 44.7 {\tiny (0.9)} & 33.9 {\tiny (1.4)}\\
         & $\mathcal{R}/\mathcal{S}$-Product & \textbf{56.6} {\tiny (1.2)} & \textbf{46.6} {\tiny (0.8)} & \textbf{39.1} {\tiny (0.9)}\\
         & \luka & 55.0 {\tiny (2.8)} & 43.8 {\tiny (9.5)} & 21.5 {\tiny (8.5)}\\
         \midrule
        5000 & $\mathcal{S}$-\godel & \textbf{55.9} {\tiny (0.5)} & \textbf{46.1} {\tiny (0.4)} & \textbf{39.9} {\tiny (0.6)}\\
        & $\mathcal{R}/\mathcal{S}$-Product & 54.3 {\tiny (0.3)} & 44.6 {\tiny (0.4)} & 36.8 {\tiny (0.3)}\\
        & \luka & 48.7 {\tiny (2.3)} & 74.1 {\tiny (9.1)} & 27.8 {\tiny (8.8)} \\
        \midrule
        25000 & $\mathcal{S}$-\godel & 54.4 {\tiny (0.2)} & \textbf{47.0} {\tiny (0.3)} & \textbf{47.2} {\tiny (0.3)}\\
        & $\mathcal{R}/\mathcal{S}$-Product & \textbf{54.9} {\tiny (0.3)} & 
        45.2 {\tiny (0.3)} & 35.2 {\tiny (0.2)}\\
        & \luka & 52.9 {\tiny (1.3)} & 44.5 {\tiny (7.1)} & 30.3 {\tiny (7.5)} \\
         \bottomrule
    \end{tabular}
    \caption{\boolean{Sum} and \boolean{Product} arithmetic properties accuracies averages (and standard deviations) from pipeline training on PAIR size 1k (noisy) labeled with \boolean{Digit} model trained on DIGIT sizes 1k, 5k, 25k.}
    \label{tab:supervised_properties_1000}
\end{table}
\begin{table}[htb!]
    \centering
    \begin{tabular}{rrccc}
    \toprule
         & & Commut. & Assoc. & Dist.  \\
         \midrule
         1000  &  $\mathcal{S}$-\godel & 
         90.0 {\tiny (0.2)} & 86.0 {\tiny (0.2)} & 82.4 {\tiny (0.5)}\\
         & $\mathcal{R}/\mathcal{S}$-Product &  \textbf{92.4} {\tiny (0.3)} & \textbf{88.9} {\tiny (0.3)} & \textbf{85.1} {\tiny (0.2)}\\
         & \luka & 85.4 {\tiny (1.9)} & 74.1 {\tiny (6.9)} & 58.6 {\tiny (7.5)} \\
         \midrule
        5000 & $\mathcal{S}$-\godel & 90.6 {\tiny (0.0)} & 88.0 {\tiny (0.1)} & 85.7 {\tiny (0.3)} \\
        & $\mathcal{R}/\mathcal{S}$-Product & \textbf{92.4} {\tiny (0.0)} & \textbf{90.1} {\tiny (0.0)} & \textbf{87.5} {\tiny (0.1)}\\
        & \luka & 88.8 {\tiny (1.2)} & 80.1 {\tiny (5.3)} & 76.5 {\tiny (6.1)} \\
        \midrule
        25000 & $\mathcal{S}$-\godel & 90.5 {\tiny (0.0)} & 88.1 {\tiny (0.0)} & 85.6 {\tiny (0.2)} \\
        & $\mathcal{R}/\mathcal{S}$-Product & \textbf{92.4} {\tiny (0.0)} & 
        \textbf{90.4} {\tiny (0.0)} & \textbf{87.9} {\tiny (0.0)} \\
        & \luka & 87.5 {\tiny (0.8)} & 80.5 {\tiny (6.0)} & 78.4 {\tiny (5.4)} \\
         \bottomrule
    \end{tabular}
    \caption{\boolean{Sum} and \boolean{Product} arithmetic properties accuracies averages (and standard deviations) from pipeline training on PAIR size 5k (noisy) labeled with \boolean{Digit} model trained on DIGIT sizes 1k, 5k, 25k.}
    \label{tab:supervised_properties_5k}
\end{table}
\begin{table}[htb!]
    \centering
    \begin{tabular}{rrccc}
    \toprule
         & & Commut. & Assoc. & Dist.  \\
         \midrule
         1000  &  $\mathcal{S}$-\godel & 94.4 {\tiny (0.0)} & 91.0 {\tiny (0.0)} & 87.2 {\tiny (0.0)}\\
         & $\mathcal{R}/\mathcal{S}$-Product & \textbf{95.5} {\tiny (0.0)} & \textbf{91.9} {\tiny (0.0)} & \textbf{87.6} {\tiny (0.0)}\\
         & \luka & 92.9 {\tiny (0.5)} & 79.9 {\tiny (6.6)} & 74.7 {\tiny (5.0)}\\
         \midrule
        5000 & $\mathcal{S}$-\godel & 95.9 {\tiny (0.0)} & 94.2 {\tiny (0.0)} & 93.0 {\tiny (0.0)} \\
        & $\mathcal{R}/\mathcal{S}$-Product & \textbf{96.8} {\tiny (0.0)} & \textbf{95.5} {\tiny (0.0)} & \textbf{93.3} {\tiny (0.0)}\\
        & \luka & 94.3 {\tiny (1.0)} & 90.0 {\tiny (3.1)} & 89.7 {\tiny (2.8)} \\
        \midrule
        25000 & $\mathcal{S}$-\godel & 96.0 {\tiny (0.0)} & 94.7 {\tiny (0.0)} & 93.5 {\tiny (0.0)}\\
        & $\mathcal{R}/\mathcal{S}$-Product & \textbf{97.4} {\tiny (0.0)} & 
        \textbf{96.8} {\tiny (0.0)} & \textbf{95.3} {\tiny (0.0)}\\
        & \luka & 96.0 {\tiny (0.2)} & 94.7 {\tiny (0.9)} & 92.9 {\tiny (3.1)} \\
         \bottomrule
    \end{tabular}
    \caption{\boolean{Sum} and \boolean{Product} arithmetic properties accuracies averages (and standard deviations) from pipeline training on PAIR size 25k (noisy) labeled with \boolean{Digit} model trained on DIGIT sizes 1k, 5k, 25k.}
    \label{tab:supervised_properties_25k}
\end{table}

\section{Full Tautologies Table and Consistencies}

Table~\ref{tab:big_tautotable} shows the entire set of tautologies we considered for the evaluation of consistencies across different relaxations. We observe that the top three relaxations preserving truth (entries equal to 1) are $\mathcal{R}$-\godel,\luka and $\mathcal{R}$-Product, however, $\mathcal{R}$-\godel does not satisfy P1 as its definition of implication is not a sub-differentiable function (see~\S\ref{consistency_subsection}).

\begin{table*}[ht]
    \centering
    \setlength{\tabcolsep}{1.1pt}
    \begin{tabular}{l|ccccc}\toprule
      Tautologies & \textbf{$\mathcal{S}$-Prod.} & \textbf{$\mathcal{S}$-\godel} & \textbf{\L{}uka.} & \textbf{$\mathcal{R}$-Pro.} & 
      \textbf{$\mathcal{R}$-\godel}\\ \midrule
    \begin{tabular}{@{}l@{}}\textbf{Axiom Schemata} \\ \quad$\boolean{P}\rightarrow (\boolean{Q}\rightarrow \boolean{P})$\end{tabular}          & \begin{tabular}{@{}c@{}}\\ 0.92\end{tabular}                    & \begin{tabular}{@{}c@{}}\\ 0.79\end{tabular} & \begin{tabular}{@{}c@{}}\\ 1\end{tabular} & \begin{tabular}{@{}c@{}}\\ 1\end{tabular} & \begin{tabular}{@{}c@{}}\\ 1\end{tabular} \\ 
	\quad$(\boolean{P}\rightarrow(\boolean{Q}\rightarrow \boolean{R}))\rightarrow((\boolean{P}\rightarrow \boolean{Q})\rightarrow(\boolean{P}\rightarrow \boolean{R}))$	  & 0.88		    & 0.75		& 0.96	 & 0.93	& 1	  \\ 
	\quad$(\neg \boolean{P}\rightarrow\neg \boolean{Q})\rightarrow(\boolean{Q}\rightarrow \boolean{P})$		  & 0.86		    & 0.75& 1	 & 0.88 & 0.79 \\ 
	\midrule
	\begin{tabular}{@{}l@{}}\textbf{Primitive Propositions} \\ \quad$(\boolean{P}\vee \boolean{P})\rightarrow \boolean{P}$\end{tabular}		  & \begin{tabular}{@{}c@{}}\\ 0.75\end{tabular}	& \begin{tabular}{@{}c@{}}\\ 0.75\end{tabular}	& \begin{tabular}{@{}c@{}}\\ 0.75\end{tabular}	 & \begin{tabular}{@{}c@{}}\\ 0.69\end{tabular} & \begin{tabular}{@{}c@{}}\\ 1\end{tabular}	  \\ 
	\quad$\boolean{Q}\rightarrow(\boolean{P}\vee \boolean{Q})$  & 0.92		    & 0.79	& 1	 & 1 & 1 \\ 
	\quad$(\boolean{P}\vee \boolean{Q})\rightarrow(\boolean{Q}\vee \boolean{P})$	 & 0.86	& 0.75		& 1	 & 1& 1	\\ 
	\quad$(\boolean{P}\vee(\boolean{Q}\vee \boolean{R} ))\rightarrow(\boolean{Q}\vee(\boolean{P}\vee \boolean{R}))$	  & 0.91	& 0.78	& 1	 & 1 & 1\\ 
	\quad$(\boolean{Q}\rightarrow \boolean{R})\rightarrow((\boolean{P}\vee \boolean{Q})\rightarrow(\boolean{P}\vee \boolean{R}))$ & 0.90		    & 0.76 & 1	 & 1 & 1 \\
	\midrule
	\begin{tabular}{@{}l@{}}\textbf{Law of excluded middle} \\ \quad$\boolean{P}\vee\neg \boolean{P}$\end{tabular}	& \begin{tabular}{@{}c@{}}\\ 0.83\end{tabular}    & \begin{tabular}{@{}c@{}}\\ 0.75\end{tabular}& \begin{tabular}{@{}c@{}}\\ 1\end{tabular} & \begin{tabular}{@{}c@{}}\\ 0.83\end{tabular}	& \begin{tabular}{@{}c@{}}\\ 0.75\end{tabular}  \\ 
	\midrule
	\begin{tabular}{@{}l@{}}\textbf{Law of contradiction} \\ \quad$\neg(\boolean{P}\wedge\neg \boolean{P})$\end{tabular}			  & \begin{tabular}{@{}c@{}}\\ 0.83\end{tabular}	& \begin{tabular}{@{}c@{}}\\ 0.75\end{tabular}	& \begin{tabular}{@{}c@{}}\\ 1\end{tabular} & \begin{tabular}{@{}c@{}}\\ 0.83\end{tabular}	& \begin{tabular}{@{}c@{}}\\ 0.75\end{tabular}  \\ 
	\begin{tabular}{@{}l@{}}\textbf{Law of double negation} \\ \quad$\boolean{P}\leftrightarrow\neg(\neg \boolean{P})$\end{tabular}  &\begin{tabular}{@{}c@{}}\\ 0.70\end{tabular}   & \begin{tabular}{@{}c@{}}\\ 0.75\end{tabular} 	& \begin{tabular}{@{}c@{}}\\ 1\end{tabular} 	 & \begin{tabular}{@{}c@{}}\\ 1\end{tabular} & \begin{tabular}{@{}c@{}}\\ 1\end{tabular}  \\ 
	\midrule
	\begin{tabular}{@{}l@{}}\textbf{Principles of transposition} \\ \quad$(\boolean{P}\leftrightarrow \boolean{Q})\leftrightarrow (\neg \boolean{P}\leftrightarrow\neg \boolean{Q})$\end{tabular}   & \begin{tabular}{@{}c@{}}\\ 0.61\end{tabular}   & \begin{tabular}{@{}c@{}}\\ 0.67\end{tabular} & \begin{tabular}{@{}c@{}}\\ 1\end{tabular}  & \begin{tabular}{@{}c@{}}\\ 0.59\end{tabular} & \begin{tabular}{@{}c@{}}\\ 0.17\end{tabular} 	  \\ 
	\quad$((\boolean{P}\wedge \boolean{Q})\rightarrow \boolean{R})\leftrightarrow((\boolean{P}\wedge\neg \boolean{R})\rightarrow\neg \boolean{Q})$ & 0.84   & 0.78	& 1	 & 0.86 & 0.65\\ 
	\midrule
	\begin{tabular}{@{}l@{}}\textbf{Laws of tautology} \\ \quad$\boolean{P}\leftrightarrow (\boolean{P}\wedge \boolean{P})$\end{tabular}  & \begin{tabular}{@{}c@{}}\\ 0.69\end{tabular}   & \begin{tabular}{@{}c@{}}\\ 0.75\end{tabular}& \begin{tabular}{@{}c@{}}\\ 0.75 \end{tabular}& \begin{tabular}{@{}c@{}}\\ 0.5\end{tabular} & \begin{tabular}{@{}c@{}}\\ 1\end{tabular}	  \\ 
	\quad$\boolean{P}\leftrightarrow (\boolean{P}\vee \boolean{P})$  & 0.69 & 0.75	& 0.75	 & 0.69 & 1  \\ 
	\begin{tabular}{@{}l@{}}\textbf{Laws of absorption} \\ \quad$(\boolean{P}\rightarrow \boolean{Q})\leftrightarrow(\boolean{P}\leftrightarrow(\boolean{P}\wedge \boolean{Q}))$\end{tabular}   & \begin{tabular}{@{}c@{}}\\ 0.66\end{tabular}  & \begin{tabular}{@{}c@{}}\\ 0.71\end{tabular}  & \begin{tabular}{@{}c@{}}\\ 0.83\end{tabular} & \begin{tabular}{@{}c@{}}\\ 0.67\end{tabular}&\begin{tabular}{@{}c@{}}\\ 1\end{tabular} \\
	\quad$\boolean{Q}\rightarrow (\boolean{P}\leftrightarrow(\boolean{P}\wedge \boolean{Q}))$ & 0.82 & 0.75 & 1 & 1 & 1 \\
	\midrule
	\begin{tabular}{@{}l@{}}\textbf{Assoc., Comm., Dist. laws} \\ \quad$(\boolean{P}\wedge(\boolean{Q}\vee \boolean{R}))\leftrightarrow((\boolean{P}\wedge \boolean{Q})\vee(\boolean{P}\wedge \boolean{R}))$\end{tabular} & \begin{tabular}{@{}c@{}}\\ 0.69\end{tabular} & \begin{tabular}{@{}c@{}}\\ 0.72\end{tabular}  & \begin{tabular}{@{}c@{}}\\ 0.90\end{tabular}  & \begin{tabular}{@{}c@{}}\\ 0.89\end{tabular}  & \begin{tabular}{@{}c@{}}\\ 1\end{tabular}  \\
	\quad$(\boolean{P}\vee(\boolean{Q}\wedge \boolean{R}))\leftrightarrow((\boolean{P}\vee \boolean{Q})\wedge(\boolean{P}\vee \boolean{R}))$ & 0.69 & 0.72 & 0.90 & 0.94 & 1 \\
	\begin{tabular}{@{}l@{}}\textbf{De Morgans Laws} \\ \quad$(\boolean{P}\wedge \boolean{Q})\leftrightarrow\neg(\neg \boolean{P}\vee \neg \boolean{Q})$\end{tabular} & 0.75 & 0.75 & 1 & 1 & 1\\
	\quad$\neg(\boolean{P}\wedge \boolean{Q})\leftrightarrow\neg(\neg \boolean{P}\vee \neg \boolean{Q})$ & 0.75 & 0.75 &1 & 1 &1\\
	\begin{tabular}{@{}l@{}}\textbf{Material excluded middle} \\ \quad$(\boolean{P}\rightarrow \boolean{Q})\vee (\boolean{Q}\rightarrow \boolean{P})$\end{tabular}& \begin{tabular}{@{}c@{}}\\ 0.97\end{tabular}  & \begin{tabular}{@{}c@{}}\\ 0.83\end{tabular}  & \begin{tabular}{@{}c@{}}\\ 1\end{tabular}  &  \begin{tabular}{@{}c@{}}\\ 1\end{tabular} & \begin{tabular}{@{}c@{}}\\ 1 \end{tabular} \\
	\bottomrule
\end{tabular}
	\caption {Degrees of consistency given by the different logic relaxations under consideration over a representative set of tautologies.}
\label{tab:big_tautotable}
\end{table*}









\end{document}